\documentclass{article}
\usepackage{ijcai23}

\usepackage{times}
\usepackage{soul}
\usepackage{url}
\usepackage[hidelinks]{hyperref}
\usepackage[utf8]{inputenc}
\usepackage[small]{caption}
\usepackage{graphicx}
\usepackage{amsthm}
\usepackage{booktabs}
\usepackage{algorithm}
\usepackage{algorithmic}
\usepackage{subcaption}
\newtheorem{prop}{Proposition}
\newtheorem{define}{Definition}
\newtheorem{lemma}{Lemma}

\usepackage{amsfonts}
\usepackage{tikz}
\usepackage{comment}
\usepackage{amsmath,amssymb}
\usepackage{microtype}
\usepackage{colortbl}
\usepackage{multirow}
\usepackage{bm}
\usepackage[switch]{lineno}

\title{IID-GAN: an IID Sampling Perspective for Regularizing Mode Collapse}

\author{
Yang Li$^*$,\ Liangliang Shi\thanks{Equal contribution. $\dag$ Correspondence author.}\And \ Junchi Yan$^\dag$
\affiliations
Department of Computer Science and Engineering,\\
MoE Key Lab of Artificial Intelligence, Shanghai Jiao Tong University\\
\emails
\{yanglily, shiliangliang, yanjunchi\}@sjtu.edu.cn
}

\begin{document}

\maketitle

\begin{abstract}
Despite its success, generative adversarial networks (GANs) still suffer from mode collapse, i.e., the generator can only map latent variables to a partial set of modes in the target distribution. In this paper, we analyze and seek to regularize this issue with an independent and identically distributed (IID) sampling perspective and emphasize that holding the IID property referring to the target distribution for generation can naturally avoid mode collapse. This is based on the basic IID assumption for real data in machine learning. However, though the source samples $\{\mathbf{z}\}$ obey IID, the generations $\{G(\mathbf{z})\}$ may not necessarily be IID sampling from the target distribution. Based on this observation, considering a necessary condition of IID generation that the inverse samples from target data should also be IID in the source distribution, we propose a new loss to encourage the closeness between inverse samples of real data and the Gaussian source in  latent space to regularize the generation to be IID from the target distribution. Experiments on both synthetic and real-world data show the effectiveness of our model.
\end{abstract}

\section{Introduction}
\label{sec:idea}
For generative models, the training data (in target space) is often assumed to be IID sampled from an unknown implicit distribution. Deep generative models often try to construct a mapping from a known distribution in source space (e.g. Gaussian distribution) to the implicit target distribution. Among popular generative schemes, e.g. Variational Auto-Encoder (VAEs)~\cite{VAEICLR14}, Generative Flow models~\cite{FlowICML15} and Generative Adversarial Networks (GANs)~\cite{goodfellow2014generative}, GAN is widely recognized for its excellence in generating images with high resolution. Mode collapse is one of the standing issues in GAN, a phenomenon that the generator tends to get stuck in a subset of modes while excluding other parts of the target distribution~\cite{CVPR2020diverse_SelfConditioned,DSGANiclr2019}, leading to poor generation diversity.

Efforts have been made to address mode collapse along two branches: 1) Achieve a better convergence between the generated distribution and the target (real data) distribution~\cite{GulrajaniNIPS17,unrollGAN}. Though the distribution convergence is recognized as the ultimate goal of generation~\cite{goodfellow2014generative}, the implicitness of the two distributions makes the desire can only be achieved by utilizing the sampling data. Thus the goal of distribution convergence can be somehow vague and imprecise. 2) Penalize the similarity of the generated images~\cite{GDPP2019ICML,MSGAN} or apply multiple generator/discriminators ~\cite{LiuCGAN16,NguyenNIPS17}. These methods promote the generation diversity in a more direct manner but often lack rigorous theoretical guarantees.

Existing mainstream GAN methods~\cite{goodfellow2014generative,ArjovskyICML17,GulrajaniNIPS17,miyato2018spectral} focus on achieving the identical generated distribution as the real one. \cite{goodfellow2014generative} shows that if given enough capacity and training time, the generated distribution converges to real distribution through GAN's optimization.  However, due to the limited capacity of the networks (e.g. catastrophic forgetting) and the implicitness of the two distributions, the convergence only holds in theory. With the short-term guidance in training, most of the methods~\cite{goodfellow2014generative,ArjovskyICML17,GulrajaniNIPS17,miyato2018spectral} tend to \textbf{sample-wise} boost the generation quality i.e. increase the probability density of the generated sample obeying the target distribution. Taking GAN as an example, discriminator $D$ distinguishes whether every generation is sampled from the identical target distribution, and generator $G$ tends to generate samples with high discriminative confidence of the current $D$. While in VAE-based or Flow-based approaches~\cite{VAEICLR14,FlowICML15}, log-likelihood is applied to increase the density of every generation. These methods rarely study the coupled relation between the generated samples to constitute a distribution, and each step the optimized sample is pulled to the center of the distribution to maintain high probability density, while the characteristics of the overall distribution e.g. variance are not explicitly controlled. This leads them to merely improve the quality of the individual generated samples, trivializing to maintain diversity as real data does. This may stem from the dilemma that a single sample can hardly estimate a distribution through nets, unless the nets possess unrestricted capability and do not exhibit catastrophic forgetting~\cite{mccloskey1989catastrophic}.

In this paper, rather than study the relation between distributions, we directly optimize the sampling relation between \textbf{generated sample batches} and the target distribution. We first make a basic yet under-exploited (in GAN literature) observation that datasets for generation are assumed to be independent and identically distributed (IID) sampled from an unknown target distribution (i.e. real distribution). We aim to achieve that the generated samples are \textbf{independently} sampled from the \textbf{identical} target distribution. If the latents $\{\mathbf{z}\}$ are IID sampled from the source, it is guaranteed that the generations $\{G(\mathbf{z})\}$ are IID samples from the generated distribution, but not necessarily IID from the target. Note the independence here is considered from the perspective of statistical test, which needs consistency over samples $\{\mathbf{x}_1,\mathbf{x}_2,\cdots,\mathbf{x}_n \}$ and the target distribution (or between the latents $\{\mathbf{z}_1,\mathbf{z}_2,\cdots,\mathbf{z}_n \}$ and source distribution). Sample batches can contain relatively more complete features characterizing a distribution, improving optimization efficiency.

To address mode collapse and get IID generation, we first define \emph{mode completeness} as the key requirement of the \emph{ideal} mapping from the source to the target. As the target distribution is implicit and we can only obtain its finite IID samples, we propose a weaker/necessary condition\footnote{Note that it will become a necessary and sufficient condition if there are an infinite amount of real data.} for IID generation based on an inverse mapping: if the real data are IID samples from target distribution, then their inverses in the source space are also IID. To achieve IID property, a common and straightforward idea is to drive the distribution of the overall inverse samples close to a standard Gaussian by certain measures. This idea differs from existing inversion-based methods~\cite{VEEGANnips17,biganiclr2017,yu2019vaegan} which learn the relationship between the individual latent sample  $\mathbf{z}$ and real data $\mathbf{x}$ through the discriminator, yet the relationship among $\{\mathbf{z}\}$ is not fully exploited. The highlights of the paper are:

{1) We take an IID perspective to mode collapse in GAN. The requirement for an ideal generator is proposed that it should satisfy the so-called mode completeness by Def. \ref{def:mode_complete}.} Limited by the finiteness of real samples, we resort to a necessary condition of mode completeness in Prop.~\ref{prop:prop1}, which requires the IID property of the inverse samples from the real data.

2) Guided by Prop.~\ref{prop:prop1}, a regularizer is devised to avoid mode collapse. We enforce the inverse sample batch to be distributed close to a standard Gaussian by Wasserstein distance, termed as Gaussian consistency loss (see Sec.~\ref{sec:gaussian_consistency_loss}). QQ-plot, Shapiro–Wilk and Kolmogorov-Smirnov statistics are used to test the IID property of the inverse Gaussian samples.

3) We show that IID-GAN outperforms baselines by different metrics on synthetic data w.r.t. the number of covered modes, quality and reverse KL divergence. Our regularization technique also performs competitively on natural images. Unsupervised disentanglement feature learning and conditional generation are also investigated. 

\section{Related Work}

Since its debut~\cite{goodfellow2014generative}, subsequent works of GANs have emerged to improve training stability and generation quality. However, issues still remain for GANs, and mode collapse is one of the most common challenges. In this section, we introduce methods that are directly or indirectly dedicated to address mode collapse.

\textbf{Improving training behavior.} The promotion of distribution convergence can naturally avoid mode collapse. Unrolled GAN~\cite{unrollGAN} presents a surrogate objective to train the generator, along with an unrolled optimization of the discriminator to improve training stability. Wasserstein GANs~\cite{ArjovskyICML17} and its variants e.g. \cite{GulrajaniNIPS17} modify the optimization objective to stabilize training. Although these methods all claim to improve the convergence between the generated and real distribution, the convergence focus more on generation quality that they tend to align every generated sample closer to the real distribution but focus less on the inter-sample relation.

\textbf{Enforcing to capture diverse modes.} Many other methods address mode collapse by penalizing similarity. GDPP~\cite{GDPP2019ICML} applies the determinantal point process theory, which imposes a penalty on the discriminator to enforce the convergence of the covariance matrix between the features of generated samples and real data. BuresGAN~\cite{de2021bures} matches the discriminator's last layer output through Bures distance. AdvLatGAN~\cite{li2022improving} mines sample pairs that tend to collapse and regularize them to maintain diversity. All these methods are concerned with the relationship among generated results. However, these methods lack stability guarantees for diverse generation, as they ignore that the generated data shall statistically maintain the same diversity as the real distribution.

\textbf{Multiple generators and discriminators.} Involving multiple generators to achieve wider coverage over the real distribution can also benefit diverse generation. In \cite{LiuCGAN16}, two coupled generators are trained with parameter sharing to learn the real distribution. The multi-agent system MAD-GAN~\cite{GhoshCVPR18} involves multiple generators along with one discriminator, and it implicitly encourages each generator to learn a partial set of modes. On the other hand, multiple discriminators are used in \cite{DurugkarICLR17} as an ensemble. AMAT~\cite{mangalam2021overcoming} introduce a novel training procedure that adaptively spawns additional discriminators to remember previous modes of generation. These methods directly seek to utilize more networks to capture more modes, but omit to reach the essence of mode collapse.

\textbf{Learn the representations by inverse mapping.}    BiGAN~\cite{biganiclr2017}, VEEGAN~\cite{VEEGANnips17}, VAE-based models~\cite{VAEICLR14} and recently proposed MGGAN~\cite{bang2021mggan}, Dist-GAN~\cite{tran2018dist} involve encoding networks (inverse mapping) of the generator to encourage the convergence between the inverse distribution of real data and source distribution. Flow-GAN~\cite{grover2018flow} utilizes a flow-based generator that allows for invertibility. However, it requires very special network design that can significantly limit the network's expressive power and inference speed.

However, similar to the previous discussion about methods improving training behavior, the inverse mapping procedure ignores the IID requirement, i.e., they merely sample-wise boost the generation quality (e.g. the probability of the sample locating in real data distribution) and fail to explicitly control the coupled relation of samples to constitute a distribution. For example, VEEGAN seeks to discriminate real/fake $(\mathbf{z},\mathbf{x})$ pairs by the discriminator while the loss is enforced to each pair. As the networks exhibit limited capacity e.g. catastrophic forgetting, the discriminator cannot accurately characterize the implicit distribution through a sequence of single sample learning, which hinders the convergence between the inverse and source distributions as shown in Fig.~\ref{fig:Generationsamples}. Under the guidance of IID generation, we take one step further to reveal the impact of IID informed inverse mapping on solving mode collapse. Specifically, Prop.~\ref{prop:prop1} calls for the entire inverse sample batch $\{G^{-1}(\mathbf{x})\}$ to be IID sampled from Gaussian, which requires considering the entire sample set, regularizing the optimization with more samples.

\section{Preliminaries and Motivation}\label{sec:sec3}
Fig.~\ref{sfig:ModeC} presents a basic interpretation of mode collapse on top of the mapping between the source and target spaces. To address mode collapse, Fig.~\ref{sfig:ModeCMap} shows a probability measure requirement for source and target space samples (i.e. probability measures for set-set pairs in the two spaces should be equal). We first define \emph{mode completeness} corresponding to the ideal generative mapping between two probability measures:
\begin{define}\label{def:mode_complete}
\emph{\textbf{(Mode Completeness)}} The probability measures $\mathbf{\alpha}$ and $\mathbf{\beta}$ are defined in the source space $\mathcal{A}$ and the target space $\mathcal{B}$, respectively. The generative mapping  $G:\mathcal{A}\to\mathcal{B}$ is defined as mode completeness from $\mathbf{\alpha}$ to $\mathbf{\beta}$ if $G$ satisfies:
\begin{equation}\label{eq:Completeness}
    \mathbf{\beta}(\mathcal{S}) = \mathbf{\alpha}(\mathbf{z}\in \mathcal{A}:G(\mathbf{z})\in \mathcal{S}\}),
\end{equation}
where $\mathcal{S}\subset\mathcal{B}$ is an arbitrary set in the the target space.
\end{define}

\begin{figure}[tb!]
  \centering
  \begin{subfigure}{0.65\linewidth}
    \includegraphics[width=1\linewidth,angle=0]{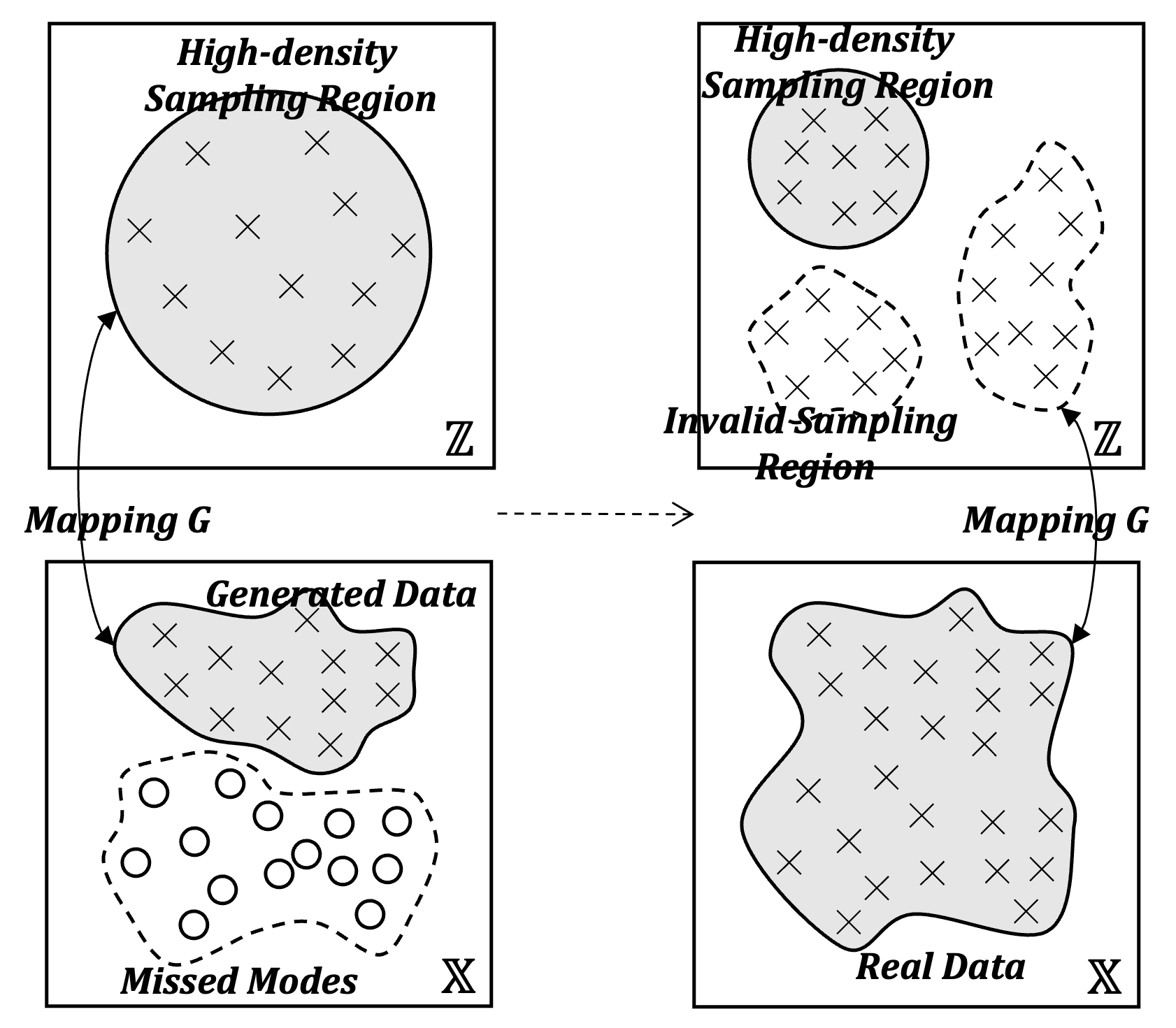}
    \caption{Mode collapse caused by real data's inverses outside high-density sampling region.}\label{sfig:ModeC}
  \end{subfigure}
  \quad
  \begin{subfigure}{0.29\linewidth}
    \includegraphics[width=0.91\linewidth,angle=0]{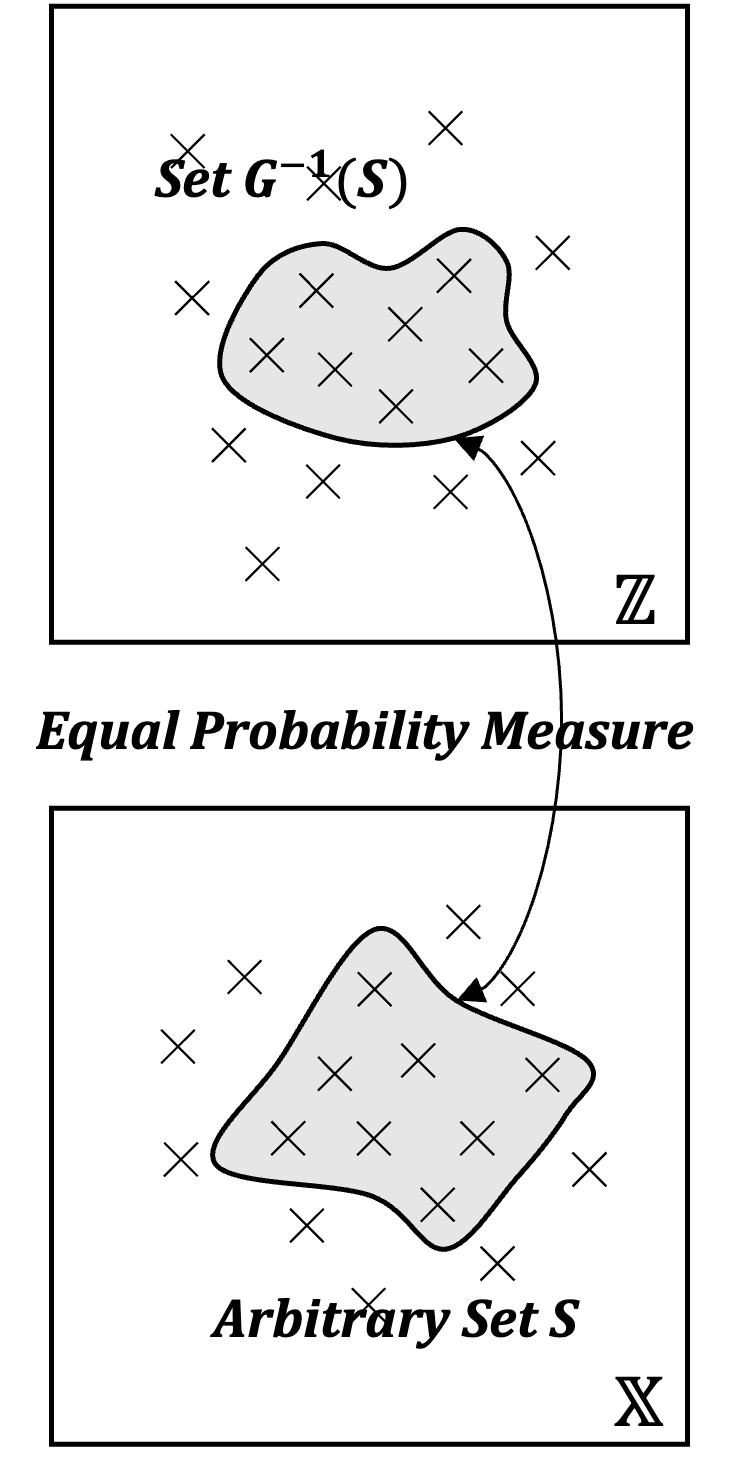}
    \caption{Mode completeness w/ equal probability measures}\label{sfig:ModeCMap}
  \end{subfigure}
  \caption{Mapping from source to target: (a) Left: high-density samples can only be mapped to part of real data i.e. mode collapse; Right: only part of real samples' inverses lie in high-density sampling region, and the inverses outside the region can hardly be sampled, leading to mode collapse. (b) Assuming the existence of the inverse of mapping $G(\cdot)$, to solve mode collapse and get IID generation, the probability measures of the set $G^{-1}(S)$ and $S$ should be equal for any set $S$ in target space. See Fig.~\ref{fig:Generationsamples} for experimental examples.}
  \label{fig:short}
\end{figure}

Eq.~\ref{eq:Completeness} defines the same operation of the push-forward operator~\cite{peyre2019computational} $\beta=G_{\#}\alpha$ in optimal transportation. So the mode completeness can also be interpreted as regularizing the mapping $G$ such that $\beta=G_{\#}\alpha$, and the operator $G_{\#}$ implies that $G$ pushes forward the mass of $\alpha$ to $\beta$~\cite{peyre2019computational}. Based on Def.~\ref{def:mode_complete}, mode collapse can be naturally avoided because the values of the probability measures given the corresponding sets based on $G$ are equal (i.e. $\alpha(\{\mathbf{z}_i\})=\beta(\{T(\mathbf{z}_i)\})$). Then given the equal probability measures, when ${\mathbf{z}_i}$ are IID sampled from $\alpha$, the generations $\{G(\mathbf{z}_i)\}$ are IID samples from $\beta$.

Mode completeness $\beta=G_{\#}\alpha$ can hardly be completely achieved since $\beta$ is complex and implicit. Considering the basic assumption that the real training data are IID sampled from $\beta$, if we assume the existence of $G^{-1}: \mathcal{B}\to\mathcal{A}$ such that $z=G^{-1}(G(z))$ and $x=G(G^{-1}(x))$ for any $z\in\mathcal{A}$ and $\mathbf{x}\in\mathcal{B}$, then we can obtain a necessary condition for mode completeness. Such a necessary condition for IID generation is formalized Proposition~\ref{prop:prop1}:
\begin{prop}\label{prop:prop1}
\emph{\textbf{ (IID for Inverse of Target Samples)}}
If the mapping $G$ satisfies mode completeness from the source probability measure $\alpha$ to the target measure $\beta$ and its inverse $G^{-1}$ exists, then given IID samples $\{\mathbf{x}^{(i)}\}_{i=1}^n$ from $\beta$, their inverses $\{G^{-1}(\mathbf{x}^{(i)})\}_{i=1}^n$ can be viewed as IID samples from $\alpha$.
\end{prop}

To prove Proposition~\ref{prop:prop1}, we first provide a proof for Lemma~\ref{prop:prop2}, which represents a more general scenario.

\begin{lemma}\label{prop:prop2}
Assume the existence of mapping $T$ which satisfies $T_{\#}\alpha=\beta$ and the existence of its inverse $T^{-1}$. Let  $\rho$ be the joint probability measure of $n$ independent copies of $\beta$, denoted as $\beta_1,\cdots,\beta_n$ and $\pi$ is the joint probability measure of $n$ independent copies of $\alpha$, denoted as $\alpha_1,\cdots,\alpha_n$. Then we have $\Tilde{T}_{\#}\pi=\rho$, where $\Tilde{T}$ is the concatenation of $n$ mapping $T$s that satisfies $\Tilde{T}(x_1,\cdots,x_n)=[T(x_1),\cdots,T(x_n)]$. Then we can obtain that $\pi$ is the joint probability measure with  $n$ \textbf{independent} probability measure $\alpha_1,\cdots,\alpha_n$.
\end{lemma}
\begin{proof}
Given $n$ measurable set $\mathcal{S}_i\subset\mathcal{B}$ and $\mathcal{S}=\mathcal{S}_1\times\mathcal{S}_2\times\cdots\times\mathcal{S}_n$, with the independence between $\{\beta_i\}_{i=1}^n$, we obtain:
\begin{equation}
     \rho(\mathcal{S})=\beta_1(\mathcal{S}_1)\beta_2(\mathcal{S}_2)\cdots\beta_n(\mathcal{S}_n)
\end{equation}
Since $\Tilde{T}_{\#}\pi=\rho$ and $T_{\#}\alpha=\beta$, we know that $\rho(\mathcal{S})=\pi(\Tilde{T}^{-1}(\mathcal{S}))$ and $\beta(\mathcal{S}_i)=\alpha(T^{-1}(\mathcal{S}_i))$. Then we have:
\begin{equation}\label{eq:joint}
     \pi(\Tilde{T}^{-1}(\mathcal{S}))=\alpha_1(T^{-1}(\mathcal{S}_1))\cdots\alpha_n(T^{-1}(\mathcal{S}_n))
\end{equation}
which indicates that $\pi$ is the joint probability measure with \textbf{independent} $\alpha_1,\cdots,\alpha_n$.
\end{proof}

The proof of Proposition~\ref{prop:prop1} can be easily obtained from  Lemma~\ref{prop:prop2} by setting $n$ measurable sets as $\mathcal{S}_1=\{x_1\}, \cdots, \mathcal{S}_n=\{x_n\}$  where $x_1, \cdots, x_n$ are $n$ IID samples from target distribution. According to Eq.~\ref{eq:joint}, we obtain
\begin{equation}
    \alpha(T^{-1}(x_1))\cdots\alpha(T^{-1}(x_n))=\pi(\Tilde{T}^{-1}(\{x_1,\cdots,x_n\}))
\end{equation}
which implies that $\{T^{-1}(x^{(i)})\}_{i=1}^n$ can be viewed as $n$ independent samples from source distribution.

Since it is commonly assumed that the real data $\{\mathbf{x}^{(i)}\}_{i=1}^n$ are independently sampled from an unknown distribution $\beta$, their inverses $\{G^{-1}(\mathbf{x}^{(i)})\}_{i=1}^n$ can be viewed as IID samples from a known distribution $\alpha$. To satisfy mode completeness, we regularize the inverse samples of real data to be closer to the IID samples from $\alpha$. Since we can hardly obtain a strict inverse mapping $G^{-1}$, we calls for an approximation of $G^{-1}$ as $F$, which is obtained by penalizing:
\begin{equation}
\label{eq:approx_F}
\begin{split}
    \mathbf{z}=F(G(\mathbf{z}))\quad \quad \mathbf{x}=G(F(\mathbf{x})).
\end{split}
\end{equation}

\section{The Proposed Approach}
\label{sec:approach}
Our loss mainly consists of three parts as shown in Fig.~\ref{fig:Structure}.
\subsection{Adversarial Learning Term: GAN Loss}
The vanilla GAN~\cite{goodfellow2014generative} consists of a discriminator $D:\mathcal{R}^d\rightarrow \mathcal{R}$ and a generator $G:\mathcal{R}^M\rightarrow \mathcal{R}^d$, which are typically embodied by deep neural networks. Given the empirical distribution $p(x)$, $D$ distinguishes whether the generated sample is from real data, while $G$ models the mapping from Gaussian sample $\mathbf{z}$ to a target space sample. The objective $V(G,D)$ is optimized for the discriminator and generator by solving the min-max:
\begin{equation}\label{eq:GAN}
    E_{\mathbf{x}\sim p_r(\mathbf{x})}\left[\log(D(\mathbf{x}))\right]+E_{\mathbf{z}\sim p(\mathbf{z})}\left[\log(1- D(G(\mathbf{z})))\right]
\end{equation}
The first term denotes the probability expectation of $\mathbf{x}$ coming from the real data distribution $p(x)$ and the second term involves the input distribution $p(\mathbf{z})$, which is embodied in this paper as a standard multi-dimensional ($M$-D) Gaussian distribution $\mathcal{N}(\mathbf{z};\mathbf{0},\mathbf{I})$. $\mathbf{I}\in \mathcal{R}^{M\times M}$ denotes the identity matrix.

%
\begin{figure*}[tb!]
\centering
\begin{minipage}[c]{0.65\linewidth}
	\centering	
	\begin{subfigure}{0.99\linewidth}
	\centering
    {\includegraphics[width=0.65\linewidth]{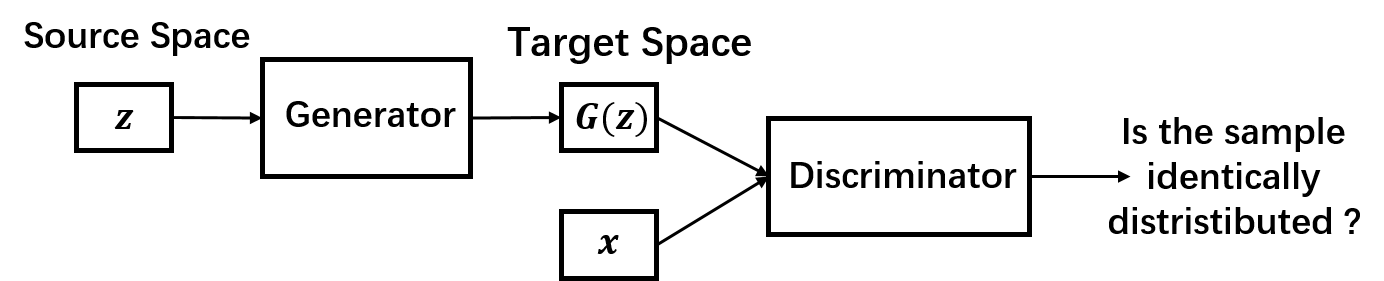}\label{sfig:GANlos}}
    \caption{Adversarial loss to generate realistic samples from the identical target distribution}
  \end{subfigure}
      \begin{subfigure}{0.99\linewidth}
      \centering
    {\includegraphics[width=0.65\linewidth]{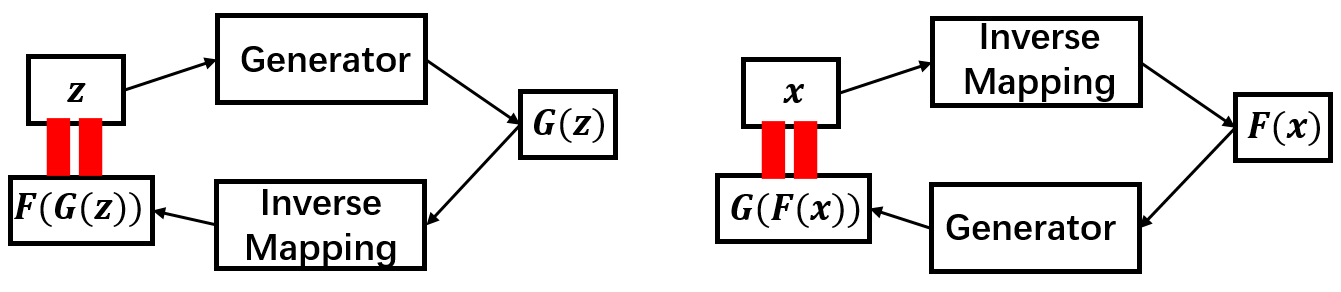}\label{sfig:inverse}}
    \caption{Cycle consistency to regularize the inverse mapping}
  \end{subfigure}
\end{minipage}
\begin{minipage}[c]{0.16\linewidth}
	\centering	
  \begin{subfigure}{0.99\linewidth}
  \centering
    {\includegraphics[width=0.67\linewidth]{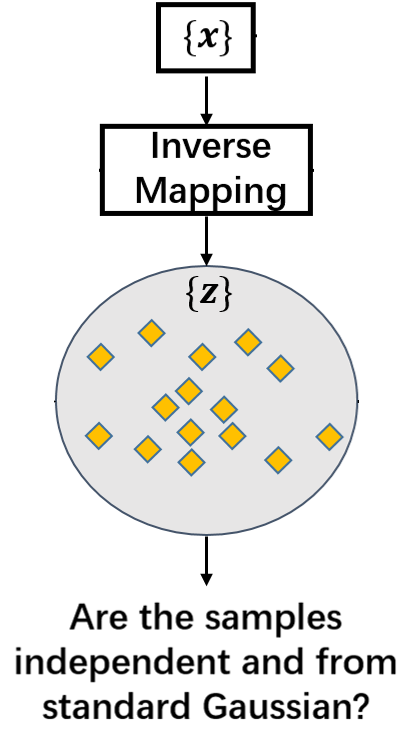}\label{sfig:Gauloss}}
    \caption{ Loss to guarantee independent sampling}
  \end{subfigure}
\end{minipage}
\caption{Generator $G$ maps $M$-D Gaussian samples to targets and $F$ inverts the target back to a source sample obeying $M$-D Gaussian.}\label{fig:Structure}
\end{figure*}

\subsection{Reconstruction Term: Cycle-Consistency Loss}
\label{sec:cycle_consistency_loss}
Prop.~\ref{prop:prop1} calls for the existence of the generator's inverse, i.e. $G^{-1}$. To achieve the approximation for the holding of $\mathbf{z}=F(G(\mathbf{z}))$ and  $\mathbf{x}=G(F(\mathbf{x}))$ where $F$ is also a  network-based mapping, a cycle-consistency loss can be an effective solution~\cite{VEEGANnips17,biganiclr2017}:
\begin{equation}
    L_{re}(G,F)=\underbrace{E_{\mathbf{z}}\Vert \mathbf{z}-F(G(\mathbf{z})) \Vert_1}_{\text{to achieve  inverse mapping}}+\lambda\underbrace{E_{x}\Vert \mathbf{x}-G(F(\mathbf{x})) \Vert_1}_{\text{to guarantee the existence}}
\end{equation}
Here $\lambda=d/M$ is the dimension ratio of $\mathbf{x}$ to $\mathbf{z}$. The first term promotes $F$ to be the inverse of $G$, which takes the reconstruction loss as the expectation of the cost for auto-encoding noise vectors~\cite{VEEGANnips17}. The second term promotes $F(\mathbf{x})\in \mathcal{R}^M$, which makes $\Tilde{\mathbf{z}}=F(\mathbf{x})$ possible to be sampled from Gaussian. Then for each $\mathbf{x}$, we can obtain the corresponding $\Tilde{\mathbf{z}}$ in $\mathcal{R}^M$  satisfying $G(\Tilde{\mathbf{z}})=\mathbf{x}$, which guarantees the existence of the real samples' inverses in source space. Note that it does not offer any help to avoid mode imbalance, since the imbalance refers to the overall distribution rather than the individual data points studied here.

\subsection{Regularizer:  Inverse Distribution Consistency}\label{subs:$M$-DGau}
\label{sec:gaussian_consistency_loss}
Recall the necessary condition for IID generation proposed in Prop.~\ref{prop:prop1}: assume that the given real data $\{\mathbf{x}^{(i)}\}_{i=1}^n$ are IID sampled from $p(x)$, then the inverse samples of the real data will be independent and obey the same distribution $p(\mathbf{z})$ (i.e. IID samples from source distribution). Thus, given a batch of real data, $\{F(\mathbf{x}^{(i)})\}_{i=1}^n$ should be closer to independent samples from standard Gaussian $p(\mathbf{z})$. 

\textbf{$M$-D Gaussian consistency loss.}
Suppose the inverse samples $\{\tilde{\mathbf{z}}^{(i)}\}_{i=1}^n\in\mathcal{R}^M$ of real data obey the Gaussian distribution $\mathcal{N}(\mathbf{z};\bm{\mu},\bm\Sigma)$ in latent space, then the maximum likelihood estimation can be formulated as:
\begin{equation}
    \Tilde{\bm{\mu}} = \frac{1}{n}\sum_{i=1}^n\Tilde{\mathbf{z}}^{(i)}, \quad
    \Tilde{\bm{\Sigma}} =  \frac{1}{n}\sum_{i=1}^n\left(\Tilde{\mathbf{z}}^{(i)}-\Tilde{\bm{\mu}}\right)^\top\left(\Tilde{\mathbf{z}}^{(i)}-\Tilde{\bm{\mu}}\right) 
\end{equation}
where $\Tilde{\bm{\mu}}\in \mathcal{R}^M$, and $\Tilde{\bm{\Sigma}}\in \mathcal{R}^{M\times M}$ is the estimated covariance. 
Our goal is to make the Gaussian $q(\mathbf{z})=\mathcal{N}(\mathbf{z};\Tilde{\bm{\mu}},\Tilde{\bm{\Sigma}})$ 
closer to the standard Gaussian $p(\mathbf{z})=\mathcal{N}(\mathbf{z};\mathbf{0},\mathbf{I})$, 
based on which we can view $\{\Tilde{\mathbf{z}}^{(i)}\}_{i=1}^n$ as IID samples from standard Gaussian $p(\mathbf{z})$. To this end, we introduce a distribution closeness loss $L_{Gau}$. For example, it can be specified as the square of Wasserstein distance of two Gaussians:
\begin{equation}
L_{Gau}=\Vert \Tilde{\bm{\mu}}\Vert^2+trace(\Tilde{\mathbf{\Sigma}}+\mathbf{I}-2\Tilde{\mathbf{\Sigma}}^{1/2})
\end{equation}
The two Gaussian distributions $q(\mathbf{z})$ and $p(\mathbf{z})$ can also be evaluated by static divergence, e.g., $p$-norm, KL-divergence. More details about loss selections are presented in Appendix~\ref{sec:methods}.

\textbf{Decoupling $M$-D Gaussian into $M$ $1$-D Gaussians.}
For large $M$ and small batch size, the training may potentially suffer from the curse of dimensionality for estimating $\bm{\Sigma}$. So we also devise a simplified consistency loss by decoupling the $M$-D Gaussian loss to the sum of the $1$-D Gaussian losses of $M$ ones. This design introduces an assumption on the covariance that the non-diagonal values are all zero. Equivalently, the inverse samples $\{\Tilde{\mathbf{z}}^{(i)}\}_{i=1}^n$ follow M $1$-D Gaussian $\mathcal{N}(\Tilde{\mathbf{z}};\bm{\mu},\bm{\sigma})$ and then we can get the estimation $\Tilde{\bm{\mu}}$ and $\Tilde{\bm{\sigma}}$.

Our goal is to make the $1$-D Gaussian distribution $q(\mathbf{z}_j)$ closer to the standard $1$-D Gaussian distribution in each dimension $j$. Similar to the $M$-D case, we can design the Gaussian loss with Wasserstein distance between two Gaussian distributions, and summing over $M$ dimensions:
\begin{equation}
    L_{Gau}= \sum_{m=1}^M\left(\bm{\Tilde{\mu}}_m^2+(\Tilde{\bm{\sigma}}_m-1)^2\right)
    \label{eq:1d_gaussian_loss}
\end{equation}
Throughout the rest of this paper, we directly call it the Gaussian loss, omitting the term \emph{consistency}.

\subsection{IID-GAN: an Overview}
By inputting samples from $p(\mathbf{z})$ together with the corresponding inverse samples $\{\Tilde{\mathbf{z}}\}$ from the real data, we optimize $G$, $D$, and $F$ by adversarial learning to push $\{\Tilde{\mathbf{z}}\}$ closer to the sampling results of $p(\mathbf{z})$. Then we can obtain the final loss as:
\begin{equation}
    \min_{G,F}\max_{D}V(G,D)+\lambda_{re}L_{re}(G,F)+\lambda_{Gau}L_{Gau}(F)
\end{equation}
where $L_{Gau}(F)$ can be specified according to different distances or divergences and $\lambda_{re},\lambda_{Gau}$ are weight parameters, which will be discussed in detail in the experiment section.

%

\textbf{Remarks for Conditional Generation.} IID-GAN can be extended to the conditioned case~\cite{mirza2014conditional} when the label of the real data $c$ is known. As for conditional IID-GAN, $(\mathbf{z},c)$ are the inputs of the generator $G$ and the real $\mathbf{x}$ is the input of $F$ for classification. The Gaussian loss is used to maintain the independence condition and learn the diversity of hidden features (e.g. the thickness for MNIST).

\textbf{Remarks for the Disentanglement View.} Many previous studies~\cite{betaVAE,FactorVAE} learn the unsupervised disentanglement representation with the assumption of independent factors, i.e. $q(\mathbf{z})=\prod_{j=1}^{M} q(\mathbf{z}_j)$. Our work presents a new viewpoint with an $M$-D Gaussian guarantee. When $q(\mathbf{z})$ approximates an $M$-D standard Gaussian, it is obvious that $\mathbf{z}$ is independent for different dimensions. However, varying $\mathbf{z}_i$ can not disentangle different modes as shown in the second column of Fig.~\ref{fig:Generationsamples}. We observe that representing the data in polar coordinates and varying the polar angle and diameter can achieve better performance for unsupervised disentangle representations.

\textbf{IID testing for Gaussian.} The regularization assumes a non-standard Gaussian and enforces the consistency of two Gaussians, which raises the doubt: Is the consistency useful for IID? We apply mathematical statistics here to test the Gaussian IID property. Specifically, we adopt the QQ-plot, Shapiro–Wilk test (SW)~\cite{shapiro1965analysis}, Kolmogorov-Smirnov test (KS)~\cite{gosset1987three} to show whether the samples are IID sampled from standard Gaussian. The results are shown in Fig.~\ref{fig:independence} and Table~\ref{tab:independence}.

\section{Experiments and Discussion}\label{sec:Experiments}
Experiments are conducted on a single GeForce RTX 3090. Synthetic data results are performed on GeForce RTX 2080Ti.
\subsection{Experiments on Synthetic Datasets}\label{sec:syndata}
Since the distribution is known for synthetic data, mode collapse can be directly measured. In line with~\cite{unrollGAN}, we simulate two synthetic datasets: 
\textbf{i) Ring dataset.} It consists of a mixture of 8 2-D Gaussians $p(\mathbf{z})$ with mean $\{(2\cos{(i\pi/4)},2\cos{(i\pi/4)})\}_{i=1}^8$ and standard deviation $0.001$. 12.5K samples are simulated from each Gaussian (100K samples in total). \textbf{ii) Grid dataset.} It consists of a mixture of 25 2-D isotropic Gaussians i.e. $p(\mathbf{z})$ with mean $\{(2i,2j)\}_{i,j=-2}^2$ and standard deviation $0.0025$. 4K samples are simulated from each Gaussian (100K samples in total).

\begin{figure*}[tb!]
	\centering
	\includegraphics[width=0.7\linewidth]{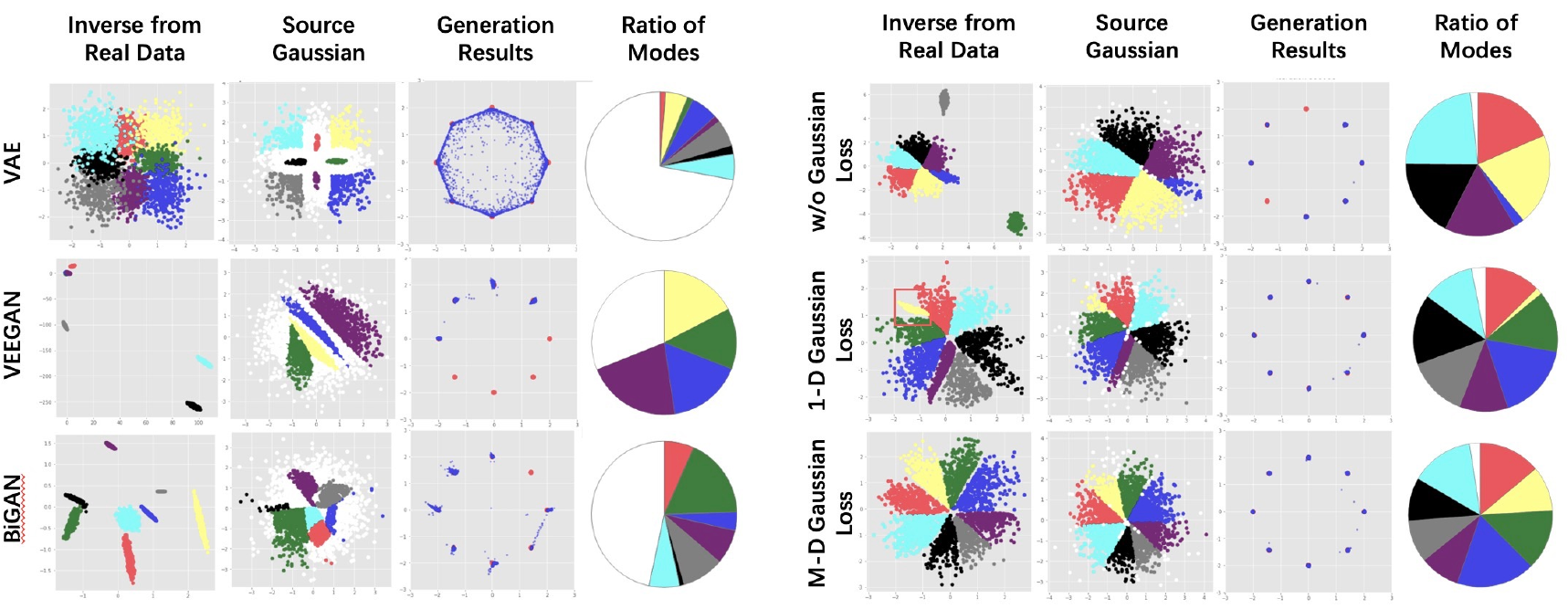}
	\caption{Example for 2-D source ($\mathbf{z}$) to 2-D target ($\mathbf{x}$) generation and inverse adopting 8-mode Ring dataset as the real training set. Left 4 columns are the results of VAE, BiGAN, VEEGAN, while the right 4 columns are the results of IID-GAN under different Gaussian consistency losses as detailed in Sec.~\ref{subs:$M$-DGau} after training 24K batches. For each half part, column 1 and column 5 show the inverse of the real target data, column 2 and column 6 show the sampled $\mathbf{z}$ from Gaussian in source space. Source points are in nine colors (8 `modes' + 1 `bad') according to their generation's mode in the target space. The pie charts show the ratio of valid generation points in different modes.}
	 \label{fig:Generationsamples}
\end{figure*}

\textbf{Metrics and network architecture.} Following \cite{unrollGAN,GDPP2019ICML}, we adopt the number of covered modes, generation quality\footnote{We follow~\cite{de2021bures}: if the generated data is within 3 times std of the Gaussian, consider it a valid generation (otherwise bad). The resulting ratio is used as the generation quality.} and reverse KL divergence as the evaluation metrics. Since in the experiment, each mode shares the same number of real samples, one can calculate the reverse KL divergence between the generated distribution and the real one~\cite{NguyenNIPS17}. Note the reverse KL divergence is not strictly defined, as $\sum_{i=1}^mp_i<1$ (there exist invalid generated points), thus it allows being negative. We adopt network architectures consisting of three linear layers with hidden dimensions of 100, 200, and 100, along with ReLU activation, for both the generator and discriminator.

\begin{table}[tb!]	
\centering 
    \resizebox{1\linewidth}{!}{
		\begin{tabular}{rcccccccc}
			\toprule
			\multirow{2}*{Models}	& \multicolumn{3}{c}{\textbf{2D-Ring}}& &\multicolumn{3}{c}{\textbf{2D-Grid}}\\\cmidrule{2-4}\cmidrule{6-8}
			& \#Mode$\uparrow$	&  Quality$\% \uparrow$& RKL$\downarrow$ &	& \#Mode$\uparrow$	& Quality$\% \uparrow$&RKL$\downarrow$ & \\	\midrule
			
			GAN
			&   3.6$\,\pm\,$0.5 & 98.8$\,\pm\,$0.6 & 0.92$\,\pm\,$0.11  &
			    & 18.4$\,\pm\,$1.6 & \textbf{98.0$\,\pm\,$0.4} & 0.75$\,\pm\,$0.25   \\
			BiGAN
			& 6.8$\,\pm\,$1.0 & 38.6$\,\pm\,$9.5 & 0.43$\,\pm\,$0.18 & 
			    & 24.2$\,\pm\,$1.2 & 83.4$\,\pm\,$2.9 & \textbf{0.26}$\,\pm\,$0.20 &   \\
			Unrolled GAN
			& 6.4$\,\pm\,$2.2 & 98.6$\,\pm\,$0.5 & 0.42$\,\pm\,$0.53 & 
			    & 8.2$\,\pm\,$1.7 & 98.7$\,\pm\,$0.6 & 1.27$\,\pm\,$0.17 &  \\
		VEEGAN
		&   5.4$\,\pm\,$1.2 & 38.8$\,\pm\,$16.7 & 0.40$\,\pm\,$0.10  &	
			    & 20.0$\,\pm\,$2.6 & 85.0$\,\pm\,$5.9 & 0.41$\,\pm\,$0.10   \\\midrule
			 IID-GAN ($1$-D)	&8.0$\,\pm\,$0.0 & 97.3$\,\pm\,$0.6& 0.18$\,\pm\,$0.06 
			    & &25.0$\,\pm\,$0.0 & 97.8$\,\pm\,$0.49 & 0.32$\,\pm\,$0.09   \\   
			 IID-GAN ($M$-D)	& \textbf{8.0$\,\pm\,$0.0} & \textbf{99.0$\,\pm\,$0.2} & \textbf{0.17$\,\pm\,$0.06} 
			    & &\textbf{25.0$\,\pm\,$0.0} & \textbf{98.0$\,\pm\,$0.4} & \textbf{0.26$\,\pm\,$0.12 }  \\ 
			    \bottomrule
		\end{tabular}
	}
\caption{Results for Ring and Grid synthetic datasets.}
\label{table:Synthetic}
\end{table}

\textbf{Quantitative results.} The comparison involve vanilla GAN~\cite{goodfellow2014generative}, BiGAN~\cite{biganiclr2017}, Unrolled GAN~\cite{unrollGAN} and VEEGAN~\cite{VEEGANnips17} on Ring and Grid datasets in Table~\ref{table:Synthetic}. IID-GAN variants are the only methods covering all modes on both datasets and IID-GAN ($M$-D) achieves the best quality and RKL performance compared to other peer methods.

\textbf{Visualization with Gaussian inverse samples.} Models involving inverse mappings (i.e. encoders) allow the visualization for the inverses of real samples. Fig.~\ref{fig:Generationsamples} present the visualization results including the inverse from real data, source Gaussian samples and generated results of VAE~\cite{VAEICLR14}, VEEGAN, BiGAN and IID-GAN variants. As shown in the left four columns, VAE can cause the overlap of the inverse samples $\mathbf{z}$, which may lead to bad generations (white points).  BiGAN and VEEGAN  learn the relations between $\mathbf{z}$ and $\mathbf{x}$ rather than the relation among the samples $\{G^{-1}(\mathbf{x})\}$, which leads to the failure of constructing 2D inverse Gaussian samples as shown in the first column. For IID-GAN, as shown in the fifth column for IID-GAN (M-D), the inverse samples are very similar to the Gaussian samples, which show the effectiveness of the regularization.  

\textbf{Visualization of generated distribution.} Fig.~\ref{fig:vis_syn} shows the generation results of IID-GAN and peer methods. IID-GAN significantly outperforms other peer methods in terms of mode coverage and better fits the target distribution.

\begin{figure}[tb!]
\centering
\begin{subfigure}{0.49\linewidth} {\includegraphics[width=1\linewidth]{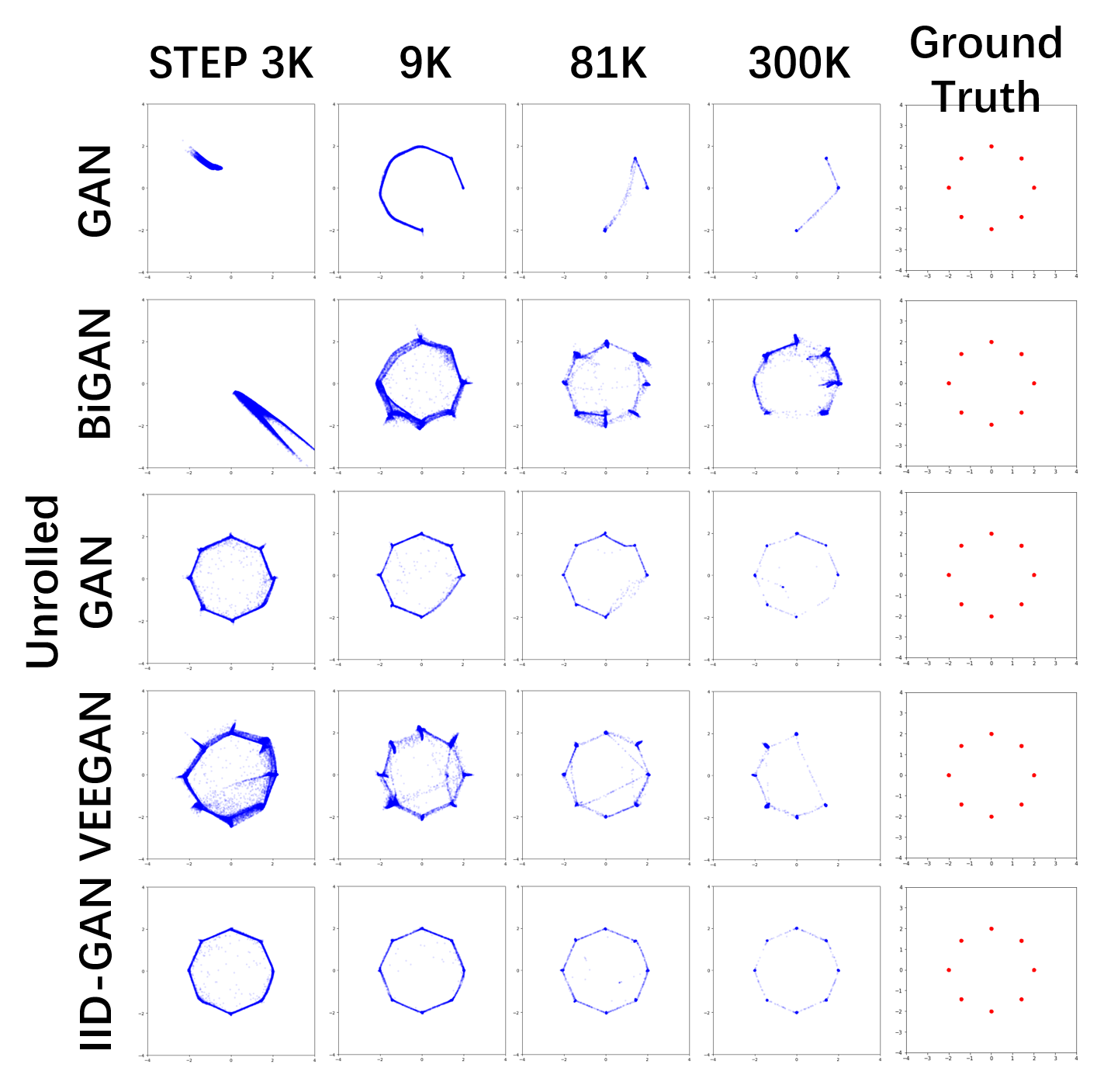}}
	\caption{Results of Ring data}
	\end{subfigure}
	\begin{subfigure}{0.49\linewidth}{\includegraphics[width=1\linewidth]{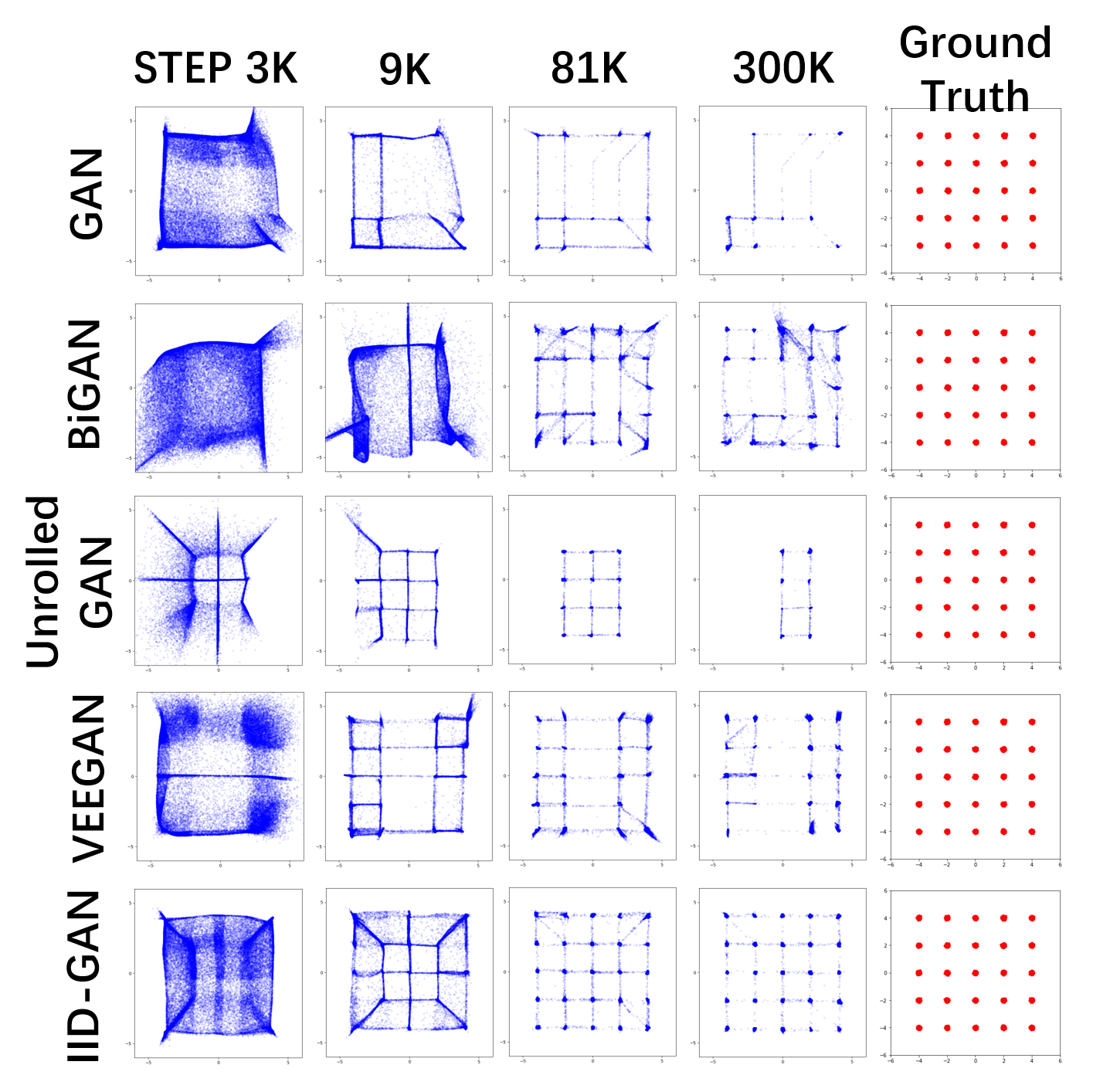}}
	\caption{Results of Grid data}
	\end{subfigure}
\caption{Generation results on Ring and Grid datasets. Blue: generated samples; Red: ground truth distribution.}
\label{fig:vis_syn}
\end{figure}

\textbf{IID test for inverse samples.}  We support IID test for Gaussian distribution on Ring dataset in  Table~\ref{tab:independence}. Shapiro-Wilk (SW) and Kolmogorov-Smirnov (KS) Statics are calculated to verify IID property. In Fig.~\ref{fig:independence}, QQ-plot are supported for IID test. Compared to other peer methods, IID-GAN's plots are closer to the red diagonal, implying that the inverse samples of IID-GAN are closer to Gaussian.

\textbf{Ablation study for Gaussian consistency loss.} The right four columns of Fig.~\ref{fig:Generationsamples} compare the results of $1$-D, $M$-D and without Gaussian consistency loss in Ring datasets. For $1$-D loss's inverse samples in source space, the yellow sample region in the red box leads to a tendency to mode collapse, while the generated modes in the $M$-D case are more uniformly scaled and more effective for solving mode collapse.

\subsection{Experiments on Real-world Data}

The experimented image datasets include MNIST~\cite{lecun1998gradient}, Stacked MNIST~\cite{unrollGAN}, CIFAR-10~\cite{cifar}, STL-10~\cite{coates2011analysis}, LSUN~\cite{yu15lsun} and CELEBA~\cite{liu2015faceattributes}.  We adopt different architectures to evaluate our model,  mainly following the previous studies~\cite{dcgan,de2021bures,miyato2018spectral}. All the compared models are trained in 100K steps. Details about the network architectures are presented in Appendix~\ref{app:realArchitectures}. We generally adopt $p$-norm distance as $M$-D Gaussian loss. Please refer to Appendix~\ref{app:details} for more details.

\begin{figure}[tb!]
	\centering	
	{\includegraphics[width=0.9\linewidth]{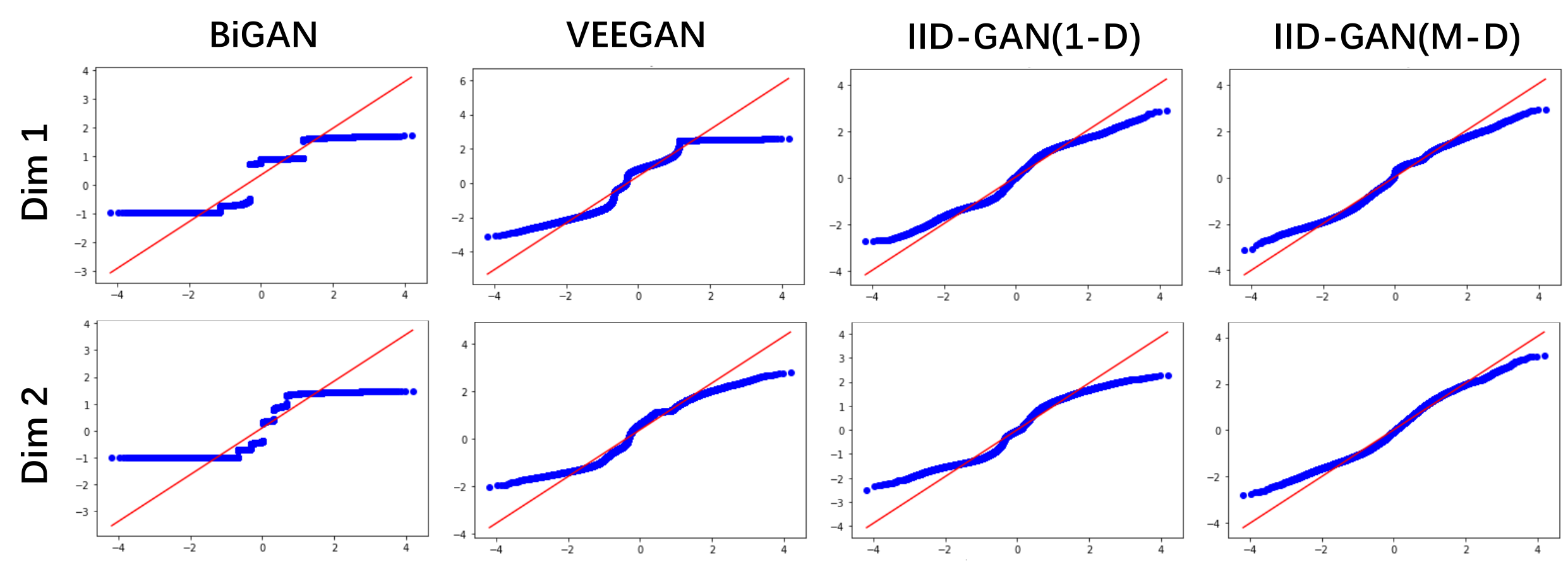}}
	\caption{QQ-plot of standard 2-D Gaussian over two dimensions. The closer to the diagonal, the closer to the Gaussian distribution.}
	\label{fig:independence}
\end{figure}

\begin{figure}[tb!]
	\centering	
	\begin{subfigure}{0.45\linewidth}
	\centering
	\includegraphics[width=1\linewidth]{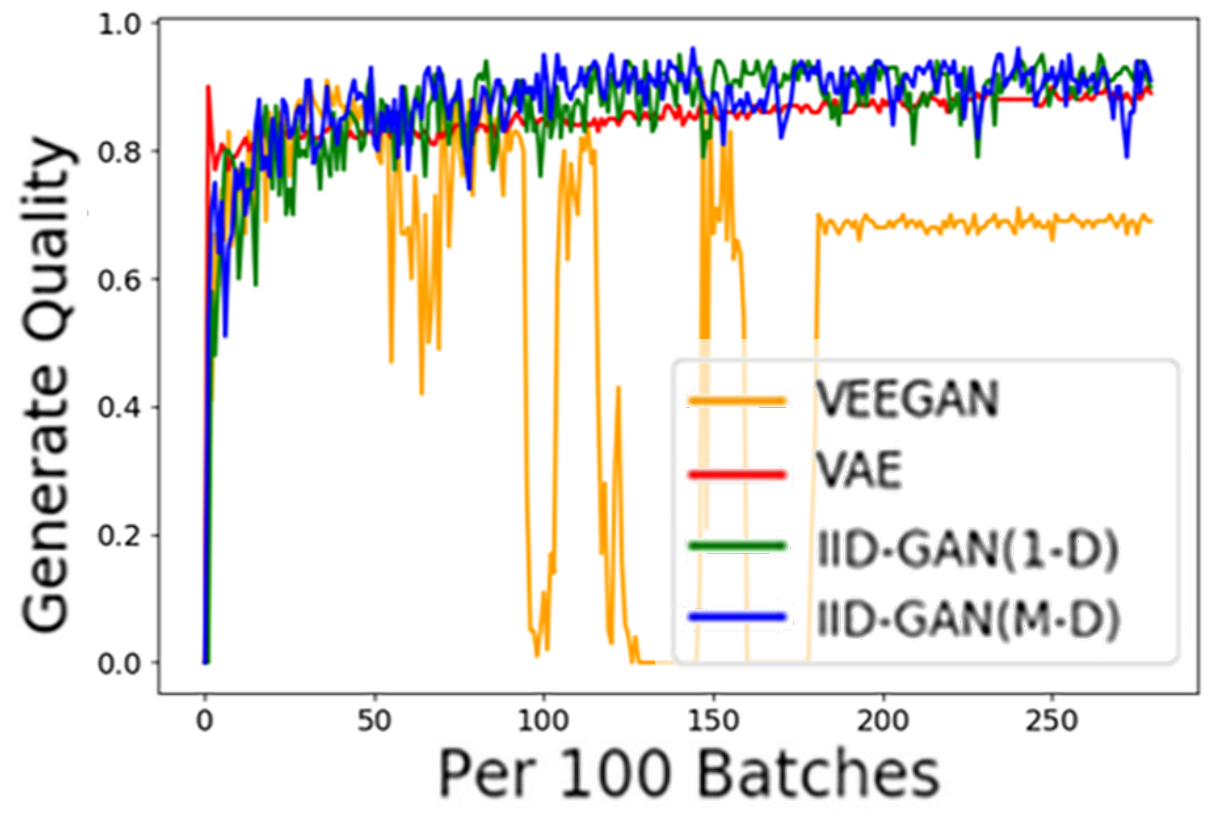}\caption{Generation Quality}\label{sfig:QMNIST}
	\end{subfigure}
	\quad\quad
	\begin{subfigure}{0.45\linewidth}
	\centering
	\includegraphics[width=1\linewidth]{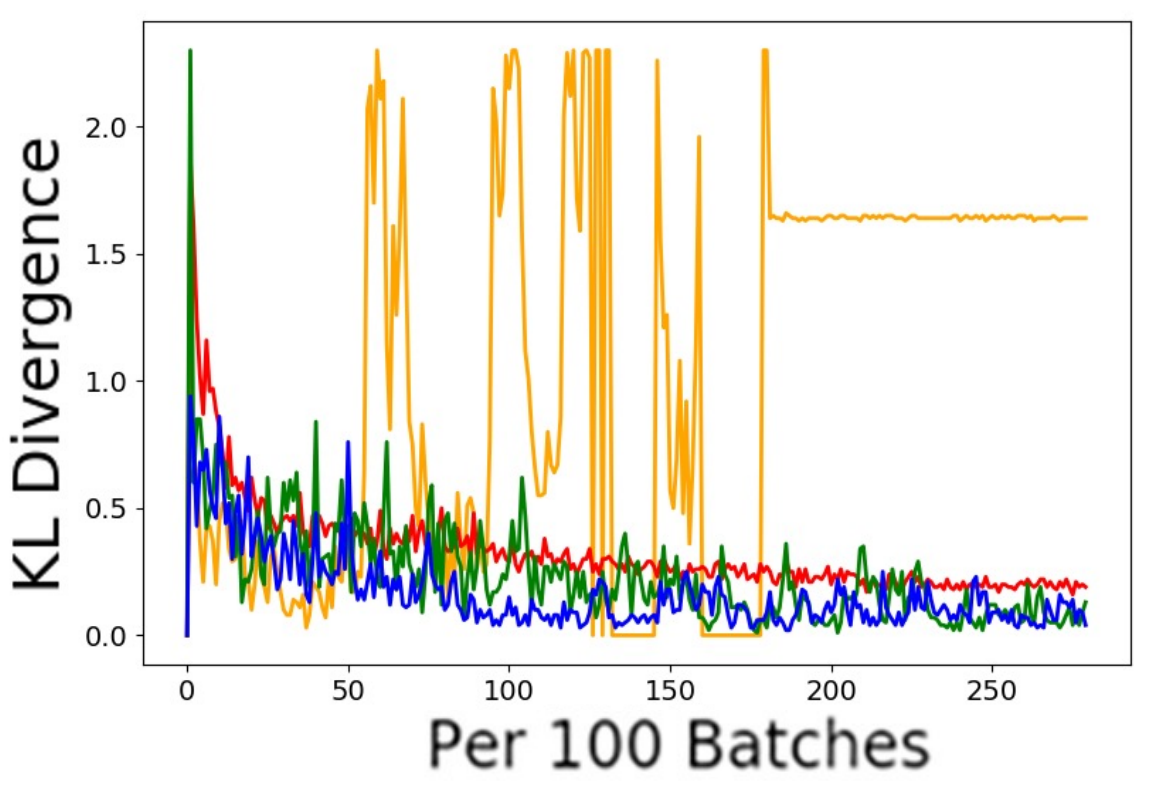}\caption{KL divergence}\label{sfig:KLMNIST}
	\end{subfigure}
	\caption{Generation quality and KL divergence (for diversity) calculated through the generations' classification results.}\label{fig:MNIST}
\end{figure}

 \begin{figure*}[t]
  \centering
	{\includegraphics[width=0.98\linewidth]{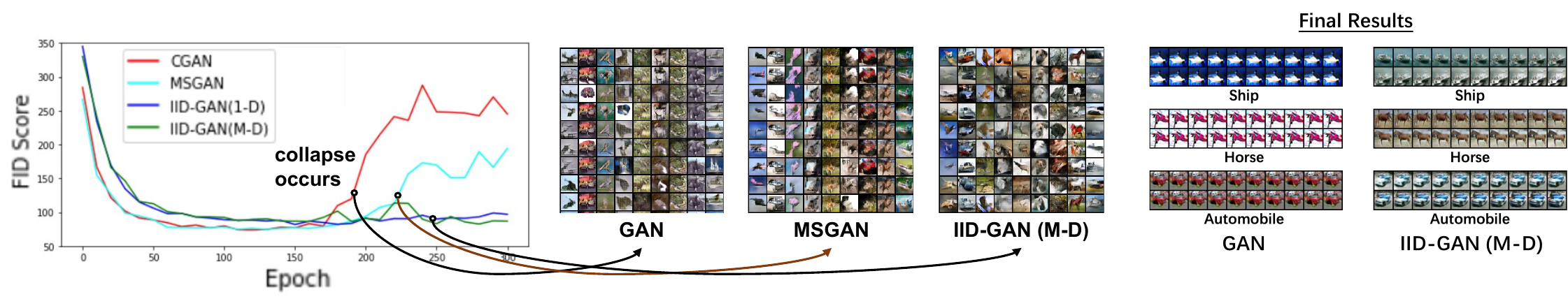}}
     \vspace{-5pt}
	\caption{Conditional results on CIFAR-10 for CGAN, MSGAN and conditional IID-GAN. The generation results of CGAN and MSGAN can easily deteriorate as training proceeds in terms of FID score, while IID-GAN performs consistently well. The final results are achieved by changing one dimension of the input latent code $z$ by increasing or decreasing the value equidistantly, producing a gradual series of images.}
	\label{fig:CIFAR10_CGAN}
\end{figure*}


\begin{figure}
\begin{minipage}[t]{0.52\linewidth}
  \centering
     \resizebox{1\linewidth}{!}{
    \begin{tabular}{rccc}
    \toprule
\multirow{2}*{Model} & 
\multirow{2}*{SW Static $\uparrow$} & 
\multicolumn{2}{c}{KS Static $\downarrow$} \\ \cmidrule{3-4}
&& $dim 1$ & $dim 2$\\
\midrule
            BiGAN
            & 0.8537    & 0.3900 & 0.1637   \\
            VEEGAN
            & 0.9567 & 0.3141   & 0.2350 \\ \midrule
            IID-GAN (1-D)  & 0.9548   & \textbf{0.0882} & 0.0866 \\
            IID-GAN (M-D)   & \textbf{0.9824} & 0.1185  & \textbf{0.0579}\\
            \bottomrule
            \end{tabular}}
      \vspace{-5pt}
    \makeatletter\def\@captype{table}\makeatother\caption{SW and KS Static evaluation for IID test on Ring.}
    \label{tab:independence}
  \end{minipage}\quad
  \begin{minipage}[t]{0.43\linewidth}
   \centering
\resizebox{1\linewidth}{!}{
       \begin{tabular}{rccccc} 
    \toprule
	\multirow{2}*{Model}	& \multicolumn{3}{c}{\textbf{Stacked MNIST}}\\\cmidrule{2-4}
	& \#Mode$\uparrow$	& KL$\downarrow$ &FID$\downarrow$\\	\midrule
	GAN
	&   392.0 &8.012&97.788 \\
	VEEGAN
	&   761.8 &2.173&86.689\\
	PACGAN
	&   992.0 &0.277&117.128 \\\midrule
	IID-GAN ($1$-D)&996.4&0.152&86.911\\
	IID-GAN ($M$-D)& \textbf{999.7}&\textbf{0.101}&\textbf{69.675}\\\bottomrule
    \end{tabular}}
         \vspace{-3pt}
    \makeatletter\def\@captype{table}\makeatother\caption{Results on Stacked MNIST.} \label{tab:mnsit}
    \label{table:Stackmnist}
   \end{minipage}
\end{figure}

\textbf{Metrics.} We use Inception Score (IS)~\cite{SalimansArxiv16} and Fréchet Inception Distance (FID)~\cite{FIDICLR17}. Following \cite{richardson2018gans}, JSD measures the distribution distance between generated images and real images through clustering, which generally capture the distribution of image modes, thus it is more diversity related. For easily distinguishable datasets like MNIST, we use the number of covered modes and reverse KL for evaluation.

\textbf{Training stability.} In Fig.~\ref{fig:MNIST}, the generation quality and KL divergence are evaluated on MNIST. The generation quality is defined as the ratio of high-confidence samples evaluated by a trained classifier. $1$-D and $M$-D IID-GAN can lead to less mode imbalance (which is evaluated by KL divergence) and higher quality than VAE, while VEEGAN is unstable and often fails to successfully generate in case of bad initialization. 

\begin{table}[tb!]
\centering
    \resizebox{1\linewidth}{!}{
     \begin{tabular}{rccccccccc}
	\toprule
	\multirow{2}*{Model}	& \multicolumn{3}{c}{\textbf{CIFAR-10}}& &\multicolumn{3}{c}{\textbf{STL-10}}\\\cmidrule{2-4}\cmidrule{6-8}
	& JSD{\scriptsize{$\times10^2$}}$\downarrow$ &  IS$\uparrow$	& FID$\downarrow$  & & JSD{\scriptsize{$\times10^2$}}$\downarrow$ &IS$\uparrow$&FID$\downarrow$\\	\midrule
	Unrolled GAN
	& 3.23$\pm$0.36 & 7.28$\pm$0.13 & 27.07$\pm$0.74  && 3.32$\pm$0.27 & 8.28$\pm$0.03 & 38.79$\pm$0.56  	\\
	VEEGAN
	& 4.83$\pm$0.33 & 7.12$\pm$0.07 & 26.94$\pm$0.94  && 5.23$\pm$0.36 & 8.39$\pm$0.05 & 37.25$\pm$0.32  \\
	MAD-GAN
	& 7.53$\pm$0.78 & 7.00$\pm$0.02 & 29.24$\pm$0.36  && 9.15$\pm$0.43 & 7.99$\pm$0.03 & 40.03$\pm$0.10	\\
	GDPP
	& 5.25$\pm$0.29 & 7.23$\pm$0.06 & 26.90$\pm$0.15  && 4.18$\pm$0.19 & 8.36$\pm$0.02 & 38.12$\pm$0.22 \\
	BuresGAN
	& 6.19$\pm$0.72 & 7.22$\pm$0.01 & 28.28$\pm$0.28  && 4.33$\pm$0.19 & 8.17$\pm$0.03 & 39.22$\pm$0.39	\\
	IID-GAN
	& \textbf{2.95$\pm$0.22} & \textbf{7.35$\pm$0.03} & \textbf{25.71$\pm$0.28}  && \textbf{3.11$\pm$0.25} & \textbf{8.43$\pm$0.04} & \textbf{36.16$\pm$0.51}  	\\
	\bottomrule
\end{tabular}}
    \vspace{-5pt}
\caption{Results on CIFAR-10 and STL-10.}\label{tab:cifar10_100}
\end{table}

\textbf{Quantitative results.} \textbf{(i)}  Table~\ref{tab:mnsit} show the \#Mode, KL divergence and FID results on Stacked MNIST. The KL divergence calculation is based on the classification results by the classifier  proposed by \cite{dieng2019prescribed}. The architecture is  DCGAN-based  in line with~\cite{dcgan}.  \textbf{(ii)} Table~\ref{tab:cifar10_100} shows the experimental results of CIFAR-10~\cite{cifar} and STL-10~\cite{coates2011analysis}. We adopt JSD, IS, FID as the metrics, including UnrolledGAN~\cite{unrollGAN}, VEEGAN~\cite{VEEGANnips17}, MAD-GAN~\cite{GhoshCVPR18}, GDPP~\cite{GDPP2019ICML} and BuresGAN~\cite{de2021bures} for comparison. To fit with advanced GAN achievements, we adopt SNGAN~\cite{miyato2018spectral} as the backbone. \textbf{(iii)} We also evaluate on larger scale datasets e.g. CELEBA~\cite{liu2015faceattributes} and LSUN Church~\cite{yu15lsun} with more recently proposed works pursuing diversity for comparison including Dist-GAN~\cite{tran2018dist}, MGGAN~\cite{bang2021mggan} and AMAT~\cite{mangalam2021overcoming}. We adopt FID, JSD as the evaluation metrics. The results are presented in Table~\ref{table:advancedArch}. \textbf{(iv)} We evaluate IID-GAN with StyleGAN2-ada~\cite{karras2020training} backbone on 64$\times$64 resized AFHQ, LSUN via FID and KID metrics for same training iterations, i.e. 10M images trained (measured in real images shown to the discriminator following StyleGAN2-ada) and show its  positive effect.

\begin{figure}[tb!]
    \centering
    \begin{subfigure}{0.4\linewidth}
	\centering
	\includegraphics[width=1\linewidth]{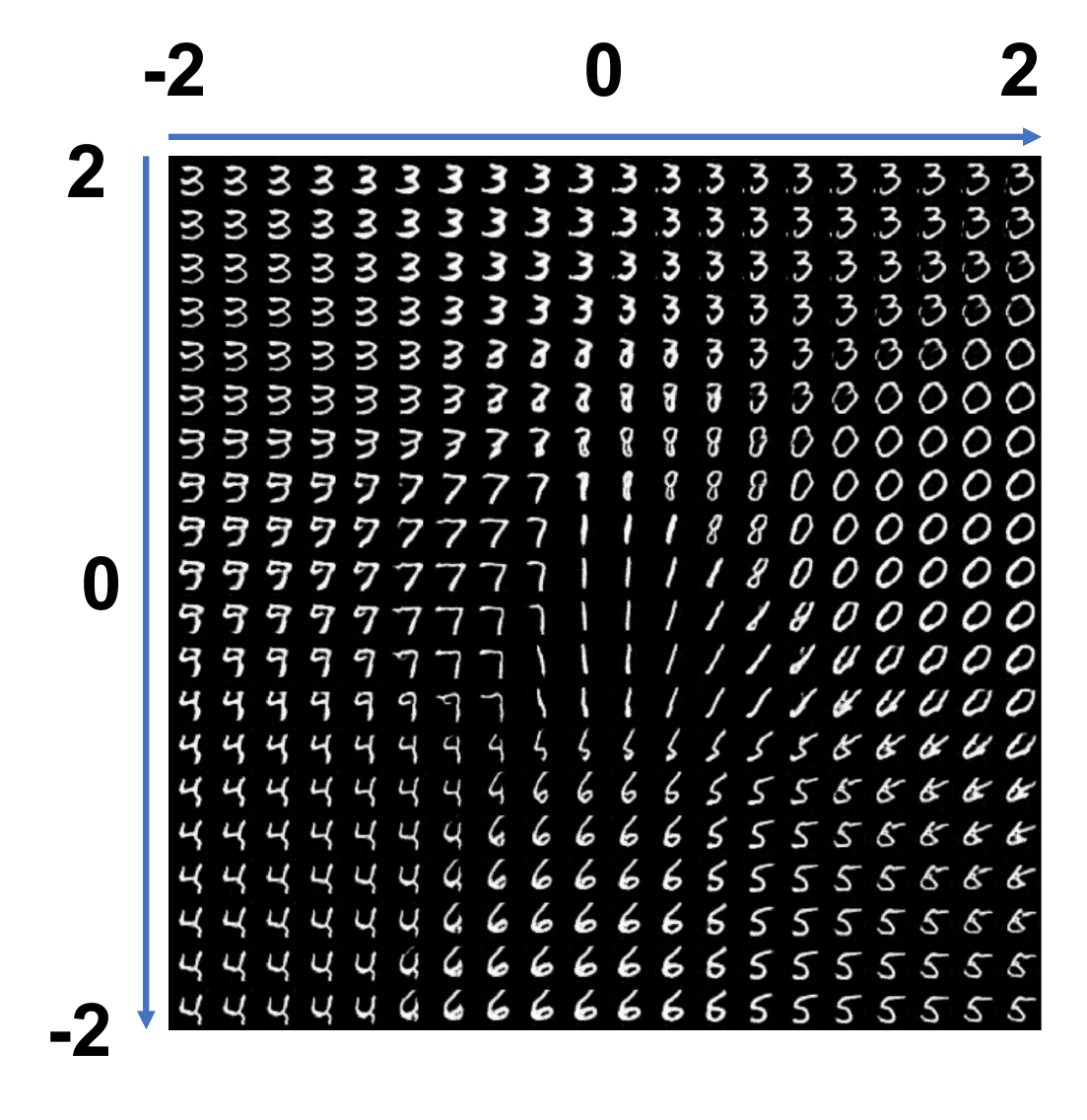}\caption{Cartedian Coordinate}\label{fig:disen1}
	\end{subfigure}\quad
\begin{subfigure}{0.4\linewidth}
	\centering
	\includegraphics[width=1\linewidth]{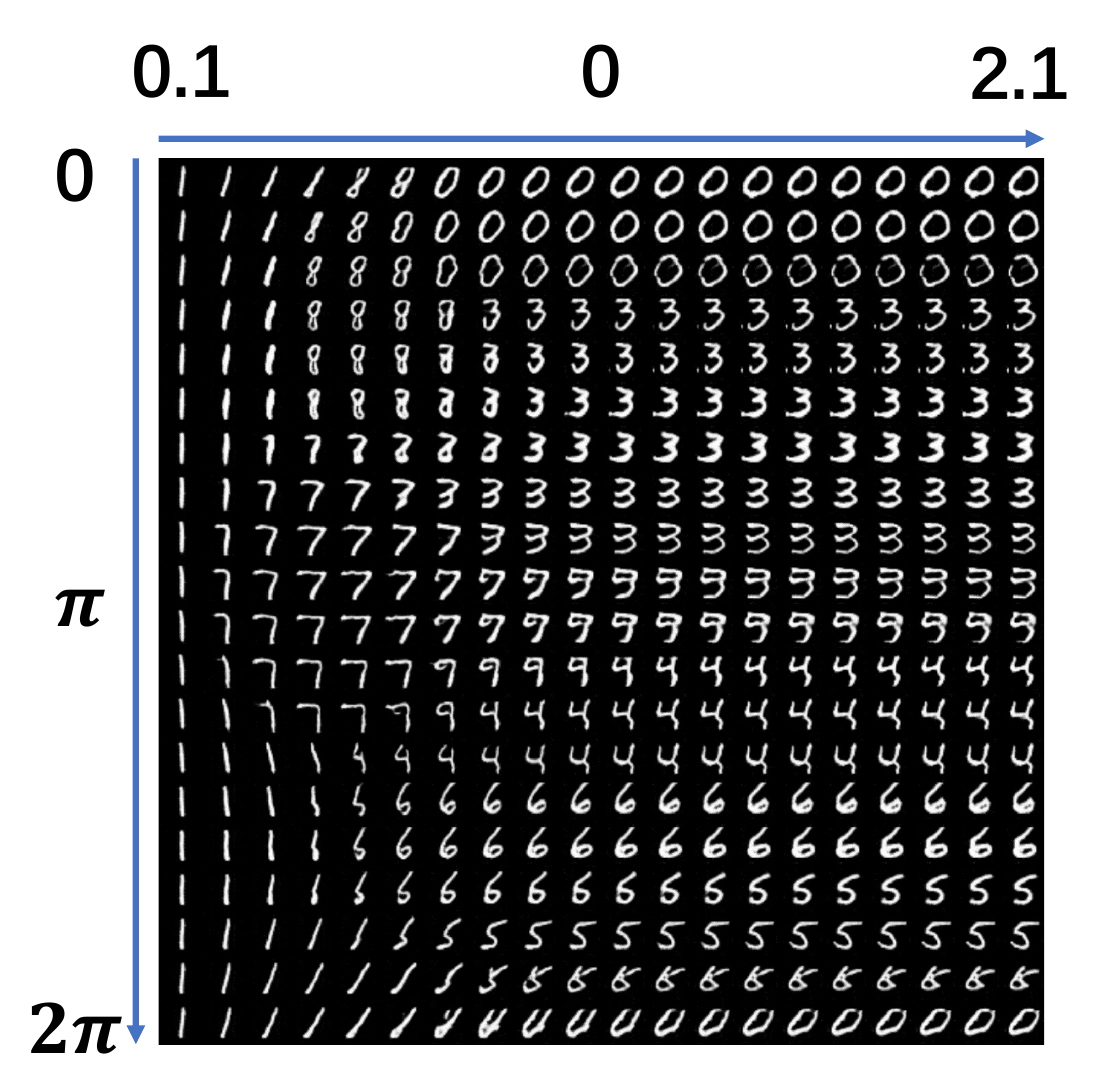}\caption{Polar Coordinate}\label{fig:disen2}
	\end{subfigure}
    \caption{Disentanglement by uniformly sampling on MNIST.}
    \label{fig:disen}
\end{figure}

\begin{table}[tb!]	
\centering
    \resizebox{0.85\linewidth}{!}{
		\begin{tabular}{cccccccc}
    \toprule
        \multirow{2}*{Method} & \multicolumn{2}{c}{\textbf{CELEBA}} & & \multicolumn{2}{c}{\textbf{LSUN}}\\\cmidrule{2-3}\cmidrule{5-6}
        &FID$\downarrow$ & JSD{\scriptsize{$\times10^2$}}$\downarrow$  &&FID$\downarrow$ & JSD{\scriptsize{$\times10^2$}}$\downarrow$ \\\midrule
        Dist-GAN
        & 37.374$\pm$0.994 & 3.604$\pm$0.202  &&
        34.491$\pm$0.487 &  5.042$\pm$0.372  \\
        MGGAN
        & 19.491$\pm$0.348  & 3.507$\pm$0.159 &&
        13.018$\pm$0.364 &  7.464$\pm$0.805\\
        AMAT
        & 29.720$\pm$0.728 & 4.789$\pm$0.158  &&
        31.585$\pm$0.923 &  4.914$\pm$0.084\\
        IID-GAN & \textbf{18.715$\pm$0.130} & \textbf{3.294$\pm$0.183} &&
        \textbf{11.755$\pm$0.451} & \textbf{4.210$\pm$0.651} \\
    \bottomrule
    \end{tabular}
 	}
 	 \vspace{-5pt}
 \caption{Evaluation (calculated at 100K steps of training) on real-world  CELEBA and LSUN, using SNGAN as backbone.}
\label{table:advancedArch}
\end{table}

\begin{table}[!htb]
    \centering
    \label{tab:1}
    \resizebox{0.9\linewidth}{!}{
    \begin{tabular}{cccccc}
    \toprule
       \multirow{2}*{Method} & \multicolumn{2}{c}{\textbf{AFHQ Cat}} && \multicolumn{2}{c}{\textbf{LSUN Church}}\\\cmidrule{2-3}\cmidrule{5-6}
        & FID$\downarrow$ & KID{\scriptsize{$\times10^3$}}$\downarrow$ && FID$\downarrow$ & KID{\scriptsize{$\times10^3$}}$\downarrow$\\\midrule
    StyleGAN2-ada & 6.386$\pm$0.503  & 1.077$\pm$0.113 && 3.530$\pm$0.021 & 1.415$\pm$0.008 \\
    IID-GAN &  \textbf{5.373$\pm$0.142}   &  \textbf{0.897$\pm$0.126} && \textbf{3.208$\pm$0.111} &  \textbf{1.317$\pm$0.095}\\
    \bottomrule
    \end{tabular}
 	}
 	\vspace{-5pt}
 	\caption{Evaluation on StyleGAN2-ada backbone.}
\end{table}

\textbf{Results on conditional generation and disentanglement.}  Fig.~\ref{fig:CIFAR10_CGAN} shows the results for conditional generation on CIFAR-10. The conditional IID-GANs are the most stable and robust model, and maintain the most diverse generation for each given category. Fig.~\ref{fig:disen} shows the visualization results in Cartesian and Polar Coordinate on MNIST. To better visualize the distribution, we use a 2-D latent space for generation.

\section{Conclusion}\label{sec:Conclusion}
To address GAN's long-standing mode collapse issue, we provide an IID sampling perspective to analyze the generation behavior and offer our new methodology guidance, which is orthogonal to existing literature. The proposed IID-GAN shows its effectiveness on both synthetic and real-world datasets.

\clearpage
\newpage
\appendix
\section*{Appendix}

\section{Gaussian Loss with Different Divergence}\label{sec:methods}
In this paper, Gaussian Consistency loss involves three different forms to reduce mode collapse:

\textbf{1) p-norm for the difference of mean and variance.} To evaluate the divergence of two Gaussian distributions $\mathcal{N}(\mathbf{z};\bm{0},\bm{I})$ and $\mathcal{N}(\mathbf{z};\Tilde{\bm{\mu}},\Tilde{\bm\Sigma})$, we first calculate the difference of the parameters of Gaussian with p-norm:
\begin{equation}
   L_{Gau}=\Vert \Tilde{\bm{\mu}} \Vert_{p}+\Vert \Tilde{\bm{\Sigma}}-\mathbf{I} \Vert_{p}
\end{equation}

\textbf{2) Wasserstein distance.}
Given two $M$-D Gaussains $p(\mathbf{z})$ and $q(\mathbf{z})$, the 2-Wasserstein distance is:
$W_2(p(\mathbf{z}),q(\mathbf{z}))=0$ if and only if $\Tilde{\mathbf{\mu}}=0$ and $\Tilde{\mathbf{\Sigma}}=\mathbf{I}$.

\textbf{3) KL divergence.}
Given $M$-D Gaussians $p(\mathbf{z})$ and $q(\mathbf{z})$, the KL divergence $ KL(p(\mathbf{z}),q(\mathbf{z}))$ can be specified as:
\begin{equation}
    L_{Gau}=\frac{1}{2}\left\{\log(\det(\Tilde{\mathbf{\Sigma}}))-M+tr(\Tilde{\mathbf{\Sigma}}^{-1})+{\Tilde{\mathbf{\mu}}}^\top{\Tilde{\mathbf{\Sigma}}}^{-1}{\Tilde{\mathbf{\mu}}}\right\}
\end{equation}


Results of IID-GAN with different Gaussion losses evaluated on synthetic dataset are presented in Fig.~\ref{fig:GaussianLoss}.

\begin{figure}[!h]
	\centering
	\begin{subfigure}{0.4\linewidth}{\includegraphics[width=1\linewidth]{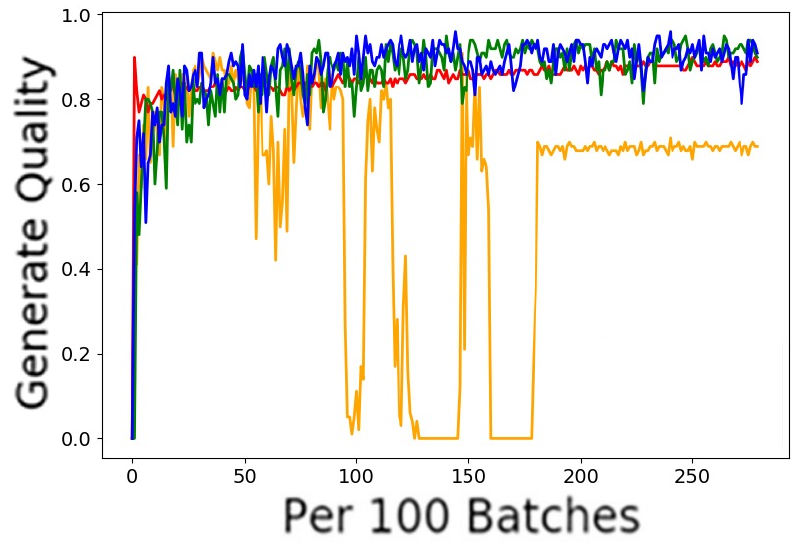}\label{sfig:QR}}\caption{Generation quality }
	\end{subfigure}
    \begin{subfigure}{0.4\linewidth}{\includegraphics[width=1\linewidth]{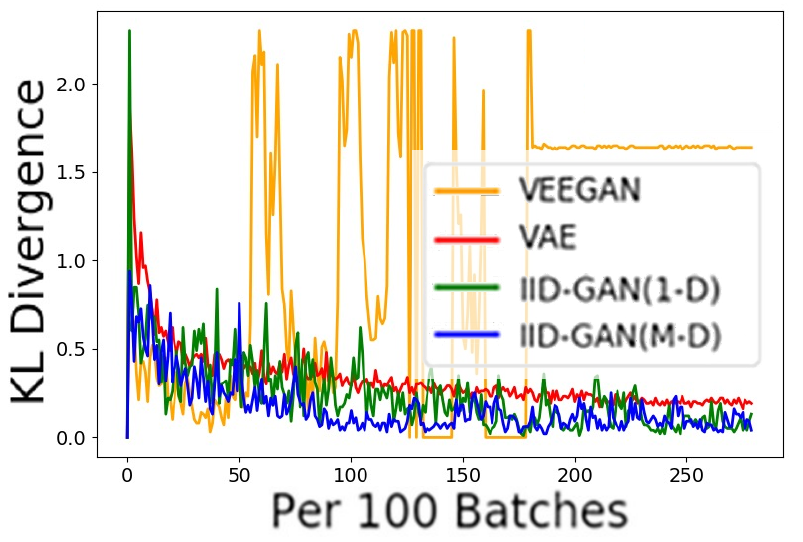}\label{sfig:KL}}\caption{Inverse KL divergence}
    \end{subfigure}
	\caption{Comparison of Gaussian consistency loss on Ring data} \label{fig:GaussianLoss}
\end{figure}

\section{Experimental Details}
\subsection{Training Details for Synthetic Data}
We set the loss weights as $(\lambda_{re},\lambda_{Gau})=(3,10)$. Training batch size is set as 100 and we conduct 300 epochs for training.  The metrics are calculated over 5 repetitions of evaluation.

\subsection{Network Architectures for Real Data}\label{app:realArchitectures}

The architectures for CIFAR-10 are presented in Table~\ref{tab:stl-g} and ~\ref{tab:stl-d}. The architectures for CELEBA and LSUN are presented in Table~\ref{tab:celeba-g} and ~\ref{tab:celeba-d}. The inverse mapping primarily follows the design of discriminator, with only the dimension of the output layer changed to match that of the latent space.

\begin{figure}
\begin{minipage}[t]{0.48\linewidth}
  \centering
     \resizebox{0.8\linewidth}{!}{
     \begin{tabular}{ccc}
     
        \toprule
        Layer& Output size\\
        \midrule
        ConvTranspose2d & $2\times2\times1024$ & \\ 
        BatchNorm2d & $2\times2\times1024$ & \\
        Relu & $2\times2\times1024$ & \\\cmidrule{1-3}
        ConvTranspose2d & $4\times4\times512$ & \\ 
        BatchNorm2d & $4\times4\times512$ & \\
        Relu & $4\times4\times512$ & \\\cmidrule{1-3}
        ConvTranspose2d & $8\times8\times256$ & \\ 
        BatchNorm2d & $8\times8\times256$ & \\
        Relu & $8\times8\times256$ & \\\cmidrule{1-3}
        ConvTranspose2d & $16\times16\times128$ & \\ 
        BatchNorm2d & $16\times16\times128$ & \\
        Relu & $16\times16\times128$ & \\\cmidrule{1-3}
        ConvTranspose2d & $32\times32\times3$ & \\ 
        Tanh & $32\times32\times3$ & \\
        \bottomrule
        \end{tabular}
     }
     
         \makeatletter\def\@captype{table}\makeatother\caption{Network Architectures of  Generator $G$ for CIFAR-10.}
     \label{tab:stl-g}
  \end{minipage}\quad
  \begin{minipage}[t]{0.45\linewidth}
   \centering
        \resizebox{0.8\linewidth}{!}{
\begin{tabular}{ccc}
        \toprule
        Layer& Output size \\
        \midrule
        SN+Conv2d & $16\times16\times64$ & \\
        LeakyRelu & $16\times16\times64$ & \\\cmidrule{1-2}
        SN+Conv2d & $8\times8\times128$ & \\ 
        LeakyRelu & $8\times8\times128$ & \\
        BatchNorm2d & $8\times8\times128$ & \\
        \cmidrule{1-2}
        SN+Conv2d & $4\times4\times256$ & \\ 
        LeakyRelu & $4\times4\times256$ & \\
        BatchNorm2d & $4\times4\times256$ & \\
        \cmidrule{1-2}
        SN+Conv2d & $2\times2\times512$ & \\ 
        LeakyRelu & $2\times2\times512$ & \\
        BatchNorm2d & $2\times2\times512$ & \\
        \cmidrule{1-2}
        faltten & $2048$ & \\ 
        linear & $1$ & \\
        \bottomrule
        \end{tabular}
        }
    \makeatletter\def\@captype{table}\makeatother\caption{Network Architecture of Inverse $D$ for CIFAR-10.}
        \label{tab:stl-d}
   \end{minipage}
\end{figure}

\begin{figure}
\begin{minipage}[t]{0.51\linewidth}
  \centering
     \resizebox{0.8\linewidth}{!}{
    
    \begin{tabular}{ccc}
    \toprule
    Layer& Output size\\
    \midrule
    Linear+Reshape&8$\times$8$\times$512\\
    \midrule
    ConvTranspose2d & $16\times16\times256$ & \\ 
    BatchNorm2d & $16\times16\times256$ & \\
    Relu & $16\times16\times256$ & \\\midrule
    ConvTranspose2d & $32\times32\times128$ & \\ 
    BatchNorm2d & $32\times32\times128$ & \\
    Relu & $32\times32\times128$ & \\\midrule
    ConvTranspose2d & $64\times64\times64$ & \\ 
    BatchNorm2d & $64\times64\times64$ & \\
    Relu & $64\times64\times64$ & \\\midrule
    Conv2d & $64\times64\times3$ & \\ 
    Tanh & $64\times64\times3$ & \\
    \bottomrule
    \end{tabular}}
     \makeatletter\def\@captype{table}\makeatother\caption{Architecture of $G$ for CELEBA and LSUN.}
     \label{tab:celeba-g}
  \end{minipage}\quad
  \begin{minipage}[t]{0.4\linewidth}
   \centering
    
        \resizebox{0.8\linewidth}{!}{
        
        \begin{tabular}{ccc}
        \toprule
        Layer& Output size \\
        \midrule
        SN+Conv2d & $64\times64\times64$ & \\
        LeakyReLU & $64\times64\times64$ & \\
        SN+Conv2d & $32\times32\times64$ & \\
        LeakyReLU & $32\times32\times64$ & \\
        \midrule
        SN+Conv2d & $32\times32\times128$ & \\
        LeakyReLU & $32\times32\times128$ & \\
        SN+Conv2d & $16\times16\times128$ & \\
        LeakyReLU & $16\times16\times128$ & \\
        \midrule
        SN+Conv2d & $16\times16\times256$ & \\
        LeakyReLU & $16\times16\times256$ & \\
        SN+Conv2d & $8\times8\times256$ & \\
        LeakyReLU & $8\times8\times256$ & \\
        \midrule
        SN+Conv2d & $8\times8\times512$ & \\
        LeakyReLU & $8\times8\times512$ & \\
        \midrule
        Faltten+Linear & 1 & \\
        \bottomrule
        \end{tabular}}
         \makeatletter\def\@captype{table}\makeatother\caption{Architecture of $D$ for CELEBA and LSUN.}
        \label{tab:celeba-d}
        
   \end{minipage}
\end{figure}

\subsection{Training Details for Real Data}\label{app:details}
\noindent\textbf{MNIST.} 
We set the loss weights as $(\lambda_{re},\lambda_{Gau})=(0.5,0.1)$. Following~\cite{dieng2019prescribed}, 
we use the 10-category classifier to divide the generated images into 11 categories. If the highest probability of the generated image's prediction is smaller than 0.75, we treat it as a bad generation and classify it to the bad class. Otherwise, its label is determined according to the highest probability. Latent space dimension is set as 2. We train the models for 300 epochs.

\noindent\textbf{StackedMNIST.}. The StackedMNIST covers 1,000 known modes, as constructed by stacking three randomly sampled MNIST images along the RGB channels in line with the practice in~\cite{unrollGAN}. We follow~\cite{VEEGANnips17} to evaluate the number of covered modes and divergence between the real and generation distributions. The weights are set as $(\lambda_{re},\lambda_{Gau})=(3,3)$. Latent space dimension is set as 100. We train the models for 50 epochs.

\noindent\textbf{CIFAR-10, STL-10, CELEBA and LSUN.} 
All models are trained for 100K steps (mini-batches). We set $(\lambda_{re},\lambda_{Gau})=(0.01,0.1)$. Latent space dimension is set as 128 and training batch size is set as 128. The metrics are calculated over 3 repetitions of evaluation. Fig.~\ref{fig:cifar10-result} presents the generation results of IID-GAN on CELEBA and LSUN Church datasets.

\begin{figure}[!ht]
    \centering 
    \begin{subfigure}{0.4\linewidth} {\includegraphics[width=1\linewidth]{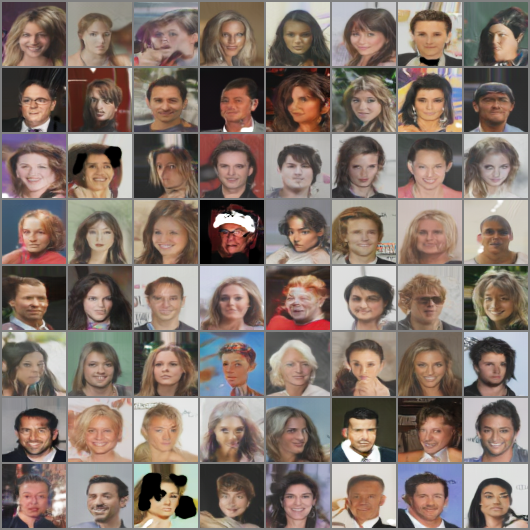}}\caption{CELEBA}
    \end{subfigure}
    \hspace{0.3cm}
    \begin{subfigure}{0.4\linewidth} {\includegraphics[width=1\linewidth]{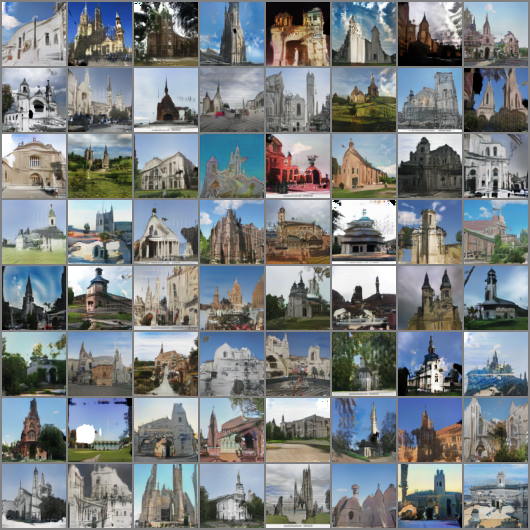}}\caption{LSUN Church}
    \end{subfigure} 
    \caption{Generation results on CELEBA and LSUN Church for unconditional IID-GAN.}
    \label{fig:cifar10-result}
\end{figure}

\section*{Contribution Statement}
\emph{Yang Li} and \emph{Liangliang Shi} contribute equally to this work. \emph{Junchi Yan} is the correspondence author. The work was partly supported by NSFC (62222607, U19B2035).

\bibliographystyle{named}
\bibliography{ijcai23}

\end{document}